\documentclass[10pt,journal,compsoc]{IEEEtran}

\usepackage{algorithm}
\usepackage{algorithmic}
\usepackage{stfloats}

\usepackage{url}
\usepackage[hidelinks]{hyperref}
\usepackage{amsmath,amsthm,amssymb,amsbsy}
\usepackage{paralist}
\usepackage{xcolor}
\usepackage{color}
\usepackage{graphicx}
\graphicspath{{./figs/}}
\usepackage{algorithm}
\usepackage{algorithmic}
\usepackage{comment}

\usepackage{fancyhdr}
\usepackage{cite}
\usepackage{cleveref}

\newtheorem{thm}{Theorem}
\newtheorem{cor}{Corollary}
\newtheorem{prop}{Proposition}
\newtheorem{defi}{Definition}

\newtheorem{lem}{Lemma}

\theoremstyle{remark}
\newtheorem{remark}{Remark}

\newcommand{\R}{\mathbb{R}}

\newcommand{\e}{\begin{equation}}
\newcommand{\ee}{\end{equation}}
\newcommand{\en}{\begin{equation*}}
\newcommand{\een}{\end{equation*}}
\newcommand{\eqn}{\begin{eqnarray}}
\newcommand{\eeqn}{\end{eqnarray}}
\newcommand{\bmat}{\begin{bmatrix}}
\newcommand{\emat}{\end{bmatrix}}

\newcommand{\vct}[1]{\boldsymbol{#1}}
\newcommand{\mtx}[1]{\boldsymbol{#1}}

\newcommand{\<}{\langle}
\renewcommand{\>}{\rangle}

\newcommand{\T}{\mathrm{T}}

\newcommand{\trace}{\operatorname{trace}}

\newcommand{\dist}{\operatorname{dist}}
\def \lg        {\langle}
\def \rg        {\rangle}

\newcommand{\domain}{\operatorname{dom}}

\DeclareMathOperator*{\argmin}{\text{arg~min}}
\DeclareMathOperator*{\argmax}{\text{arg~max}}
\def \st {\operatorname*{subject\ to\ }}

\newcommand{\calC}{\mathcal{C}}

\newcommand{\ve}{\vct{e}}

\newcommand{\vs}{\vct{s}}

\newcommand{\vu}{\vct{u}}
\newcommand{\vv}{\vct{v}}

\newcommand{\vx}{\vct{x}}
\newcommand{\vy}{\vct{y}}
\newcommand{\vz}{\vct{z}}
\newcommand{\vzero}{\vct{0}}

\newcommand{\mA}{\mtx{A}}
\newcommand{\mB}{\mtx{B}}

\newcommand{\mD}{\mtx{D}}

\newcommand{\mG}{\mtx{G}}
\newcommand{\mH}{\mtx{H}}

\newcommand{\mM}{\mtx{M}}
\newcommand{\mN}{\mtx{N}}

\newcommand{\mS}{\mtx{S}}

\newcommand{\mU}{\mtx{U}}
\newcommand{\mV}{\mtx{V}}
\newcommand{\mW}{\mtx{W}}
\newcommand{\mX}{\mtx{X}}
\newcommand{\mY}{\mtx{Y}}

\newcommand{\mLambda}{\mtx{\Lambda}}

\newcommand{\mId}{{\bf I}}

\newcommand{\mzero}{{\bf 0}}

\graphicspath{{./figs/}}

\newlength{\imgwidth}
\newcommand{\revise}[1]{{\color{black}{#1}}}

\newboolean{twoColVersion}
\setboolean{twoColVersion}{false}
\newcommand{\twoCol}[2]{\ifthenelse{\boolean{twoColVersion}} {#1} {#2} }

\usepackage[framemethod=tikz]{mdframed}

\usepackage{framed}
\usepackage{pifont}
\usepackage{enumitem}
\title{A Provable Splitting Approach for \\ Symmetric Nonnegative Matrix Factorization}

\author{Xiao Li,  Zhihui Zhu, Qiuwei Li, and Kai Liu
\thanks{X. Li is with School of Data Science, The Chinese University of Hong Kong, Shenzhen and is with Shenzhen Institute of Artificial Intelligence and Robotics for Society (AIRS) (e-mail: lixiao@cuhk.edu.cn).}
\thanks{Z. Zhu is with the Department of Electrical and Computer Engineering at the University of Denver. (corresponding author. e-mail: zhihui.zhu@du.edu).}
\thanks{Q. Li is with the Decision Intelligence Lab, Damo Academy, Alibaba Group US. (e-mail: liqiuweiss@gmail.com).}
\thanks{K. Liu is with Computer Science Division,  Clemson University. (e-mail: liukaizhijia@gmail.com).}
\thanks{Part of this work appears in a conference proceeding \cite{zhu2018dropping}.  
This work significantly extends the preliminary conference version. we have three new contributions compared to the conference version: 1) We propose the new accelerated SymHALS (A-SymHALS) algorithm. 2) We provide a unified convergence analysis that guarantees sequence convergence for all of our algorithms. The standard convergence analysis technique used in \cite{zhu2018dropping} cannot be directly applied due to the specific updating scheme of A-SymHALS. We provide a new analysis framework that is different from the existing standard one and is of independent interest; see the last paragraph of Section 1.1 for details.  3) We propose a new adaptive strategy for updating the penalty parameter $\lambda$. This adaptive strategy could be crucial for the practical use of our algorithms for solving symmetric NMF as it avoids hyper-parameter tuning. In addition to the above three important new contributions, we also conduct more experiments in this work. 
}
}

\begin{document}
	
\IEEEtitleabstractindextext{%
\begin{abstract}
The symmetric Nonnegative Matrix Factorization (NMF), a special but important class of the general NMF, has found numerous applications in data analysis such as various clustering tasks. Unfortunately, designing fast algorithms for the symmetric NMF is not as easy as for its nonsymmetric counterpart, since the latter admits the splitting property that allows state-of-the-art alternating-type algorithms. To overcome this issue, we first split the decision variable and transform the symmetric NMF to a penalized nonsymmetric one, paving the way for designing efficient alternating-type algorithms.  We then show that solving the penalized nonsymmetric reformulation returns a solution to the original symmetric NMF. Moreover, we design a family of  alternating-type algorithms and show that they all admit strong convergence guarantee: the generated sequence of iterates is convergent and converges at least sublinearly to a critical point of the original symmetric NMF.  Finally, we conduct experiments on both synthetic data and real image clustering to support our theoretical results and demonstrate the performance of the alternating-type algorithms. 
\end{abstract}

	\begin{IEEEkeywords}
		Symmetric nonnegative matrix factorization, convergence, image clustering, alternating minimization.
\end{IEEEkeywords}}

\maketitle
\IEEEdisplaynontitleabstractindextext

\section{Introduction}

The general nonsymmetric  Nonnegative Matrix Factorization (NMF) is referred to the following problem: given a matrix $\mY\in \R^{n\times m}$ and a factorization rank $r$, solve
\e
\begin{split}
	&\min_{\mU\in\R^{n\times r},\mV\in\R^{m\times r}}  \ \frac{1}{2}\|\mY - \mU\mV^\T\|_F^2 \\
	&\st \mU\geq \vzero, \mV \geq \vzero,
\end{split}
\label{eq:NMF}\ee
where $\mU\geq \vzero$ means each element in $\mU$ is nonnegative. NMF has been successfully used in the applications of face feature extraction \cite{lee1999learning,guillamet2002non},  document clustering \cite{shahnaz2006document,cai2010graph,liu2011constrained}, \revise{image clustering~\cite{li2019discrimination,wang2020robust},} music analysis \cite{fevotte2009nonnegative},  source separation~\cite{ma2014signal} and many others \cite{gillis2014and}. Because of the ubiquitous applications of NMF,  many efficient algorithms have been proposed for solving problem \eqref{eq:NMF}.  Well-known algorithms include multiplicative update algorithm ~\cite{lee2001algorithms}, \revise{Projected Gradient Descent (PGD) \cite{lin2007projected}, Alternating Nonnegative Least Squares (ANLS) \cite{kim2008toward}, and Hierarchical Alternating Least Squares (HALS) \cite{cichocki2009fast}}. In particular, ANLS (which uses the block principal pivoting algorithm to very efficiently solve the nonnegative least squares) and HALS achieve the state-of-the-art performance.

One special but important class of NMF, called symmetric NMF, requires  the two factors $\mU$ and $\mV$ to be identical, i.e., it factorizes a symmetric matrix $\mX\in\R^{n\times n}$ by solving
\e
\min_{\mU\in\R^{n\times r}} \frac{1}{2}\|\mX - \mU\mU^\T\|_F^2,\quad \st \mU\geq \vzero.
\label{eq:SNMF}\ee
By contrast, \eqref{eq:NMF} is referred to as the general nonsymmetric NMF. Symmetric NMF has its own applications in data analysis, machine learning  and signal processing~\revise{\cite{ding2005equivalence,he2011symmetric,kuang2015symnmf,luo2017symmetric,hu2020augmented}}. In particular the symmetric NMF is equivalent to the classical $K$-means  kernel clustering  in \cite{ding2005equivalence} and it is inherently suitable for clustering nonlinearly separable data from a symmetric similarity matrix~\cite{kuang2015symnmf}.


At first glance, since \eqref{eq:SNMF}  has only one variable, one may think it is easier to be solved  than \eqref{eq:NMF} or at least it can be solved by directly utilizing  efficient algorithms developed for \eqref{eq:NMF}. However, the state-of-the-art alternating-type algorithms (such as ANLS and HALS) for solving the general nonsymmetric NMF utilize the splitting property of the decision variables in \eqref{eq:NMF} and thus can not be used for tackling \eqref{eq:SNMF}. On the other hand, first order method like PGD when used to solve  \eqref{eq:SNMF} suffers from  slow convergence. 

\subsection{Main Contributions} \label{sec:contributions}
In this paper, we compute the symmetric NMF  by considering a variable splitting method, which reformulates our problem to a penalized nonsymmetric NMF. This new nonsymmetric reformulation enables us to design efficient alternating-type algorithms  for solving the original symmetric NMF.   
The main contributions of this paper are summarized as follows.
\begin{itemize}[leftmargin=*]
	\item Motivated by the splitting property exploited in ANLS and HALS algorithms, we split the quadratic form on $\mU$ in the symmetric NMF into two different factors and transform symmetric NMF to a penalized nonsymmetric NMF, i.e., 
\e
\begin{split}
  &\min_{\mU,\mV} g(\mU,\mV) = \frac{1}{2}\|\mX - \mU\mV^\T\|_F^2 + \frac{\lambda}{2}\|\mU - \mV\|_F^2 \\
  &  \st \mU\geq \vzero, \mV \geq \vzero,
\end{split}
\label{eq:SNMF by reg}\ee
	where the penalty term $\|\mU - \mV\|_F^2$ is introduced to force the two factors identical and $\lambda>0$ is the penalty parameter. Our first main contribution is to guarantee that  with a sufficiently large but \emph{finite} $\lambda$, \emph{any critical point} $(\mU^\star,\mV^\star)$ of \eqref{eq:SNMF by reg} that has bounded energy (where the upper bound depends on $\lambda$) satisfies that $(i)$ $\mU^\star = \mV^\star$ and $(ii)$ $\mU^\star$ is a critical point of the symmetric NMF \eqref{eq:SNMF}. The result is surprising in the sense that classical result of the methods of Lagrangian multipliers suggests that the two factors will be identical only when $\lambda$ tends to \emph{infinity} \revise{since the quadratic penalty is not an exact penalty function~\cite[Theorem 17.1]{nocedal2006numerical}}.
	\item We further show that  any algorithm  possessing descent and convergence properties for solving  \eqref{eq:SNMF by reg} is guaranteed to yield a critical point of \eqref{eq:SNMF}, provided that the penalty parameter $\lambda$ is properly chosen.  In summary, this observation suggests that the symmetric NMF can be provably solved by instead addressing the  nonsymmetric NMF \eqref{eq:SNMF by reg} which enjoys the splitting property within the factors $\mU$ and $\mV$.  
	\item Motivated by ANLS, HALS, and accelerated HALS \cite{gillis2012accelerated}, we then design a family of alternating-type algorithms---namely \revise{the Symmetric Alternating Nonnegative Least Squares (SymANLS; see \Cref{alg:ANLS}), the Symmetric Hierarchical Alternating Least Squares (SymHALS; see \Cref{alg:HALS}), and the Accelerated Symmetric Hierarchical Alternating Least Squares (A-SymHALS; see \Cref{alg:Accelerated HALS})---to solve the penalized nonsymmetric NMF  \eqref{eq:SNMF by reg}}.
	Our third  contribution is to provide  a unified rigorous convergence analysis for these three algorithms. By exploiting the specific structure of \eqref{eq:SNMF by reg}, we show that our proposed algorithms are guaranteed to sequentially decrease the objective function in  \eqref{eq:SNMF by reg}  even without any proximal terms or any additional boundedness constraints on $\mU$ and $\mV$. Consequently,  we establish the point-wise  \emph{sequence convergence} to a critical point $(\mU^\star,\mV^\star)$  of \eqref{eq:SNMF by reg}, where the convergence rate is at least \emph{sublinear}. Finally,  it is worth mentioning that  the disciplined  Kurdyka-Lojasiewicz convergence analysis framework \cite{attouch2010proximal,bolte2014proximal} cannot be directly applied to A-SymHAL due to its acceleration scheme, i.e., it updates one variable multiple times before moving to the other variable. We generalize this convergence analysis framework to accommodate this accelerate scheme, which is of independent interest. 
\end{itemize}

\subsection{Related Work}\label{sec:related work}
 Due to slow convergence of PGD for solving the symmetric NMF, several algorithms have been proposed, either in a direct way or similar to \eqref{eq:SNMF by reg} by splitting the two factors. The authors in \cite{vandaele2016efficient} proposed an alternating algorithm that
cyclically optimizes over each element in $\mU$ by solving a nonnegative constrained nonconvex univariate fourth order polynomial minimization.  A  quasi newton second order method was used in \cite{kuang2015symnmf} to directly solve the symmetric NMF optimization problem \eqref{eq:SNMF}. However, both the element-wise updating approach and the second order method are computationally expensive for large scale applications.  In \cite{he2011symmetric}, the authors designed an accelerated multiplicative update algorithm, while in \cite{huang2013non} the authors proposed an \revise{Singular Value Decomposition (SVD)}-based algorithm that iteratively approximates the symmetric NMF.  Nevertheless, the experiments in \Cref{sec:experiments} indicate that they tend to get stuck at local minima with large fitting errors for noisy data.

The idea of solving symmetric NMF by targeting  the penalized nonsymmetric NMF \eqref{eq:SNMF by reg}  also appears \emph{heuristically} in \cite{kuang2015symnmf}. The ANLS algorithm is used in \cite{kuang2015symnmf} for solving \eqref{eq:SNMF by reg}, but without  any formal analysis for the convergence and the question that whether solving \eqref{eq:SNMF by reg} returns a solution of \eqref{eq:SNMF}. \revise{
The work \cite{borhani2016fast,lu2017nonconvex,hu2020augmented} considered an augmented Lagrangian formulation of \eqref{eq:SNMF} that also enjoys the splitting property as in \eqref{eq:SNMF by reg} by splitting the quadratic form $\mU\mU^\top$ into $\mU\mV^\top$ and introducing an equality constraint (i.e., $\mU = \mV$), and utilized the Alternating Direction Method of Multipliers (ADMM) or its variants  to tackle the corresponding problem. Unlike the alternating-type algorithms for \eqref{eq:SNMF by reg} that will be proved to have sequence convergence in \Cref{sec:fast algorithms}, however, the ADMM is only guaranteed to have a  subsequence convergence, even with an additional proximal term\footnote{\revise{In $k$-th iteration, a proximal term (e.g., $\|\mU-\mU_{k-1}\|_F^2$) is added to the objective function when updating $\mU$ in \cite{lu2017nonconvex} and when updating both $\mU$ and $\mV$ in \cite{hu2020augmented}.}} and 
assumption on the boundedness of the iterates \cite{hu2020augmented} or
a constraint on the boundedness of columns of $\mU$ \cite{lu2017nonconvex}, rendering the problem hard to solve.}

Our work is also closely related to recent advances in convergence analysis for alternating minimization algorithms. The work \cite{attouch2010proximal} established sequence convergence for general alternating minimization algorithms with an additional proximal term and a boundedness assumption on the iterates.
When specified to NMF, as pointed out in~\cite{huang2016flexible}, with the aid of an additional proximal term as well as an additional constraint bounding the factors, the sequence convergence of ANLS and HALS can be established from~\cite{attouch2010proximal,razaviyayn2013unified}. 
Although the convergence of these algorithms are observed without the proximal term and bounded constraint (which are indeed not used in practice), these are in general necessary to formally show the convergence of the algorithms. By contrast,  without any additional constraint, the presence of the penalty term $\|\mU-\mV\|_F^2$ allows us to show that $(i)$ our proposed algorithms admit the so-called sufficient decrease property, and consequently, $(ii)$ the iterates generated by our algorithms are indeed bounded along the iterations. These observations then guarantee the sequence convergence of the practical  algorithms without those additionals constraint or proximal terms.  


\section{Transforming Symmetric NMF to  Penalized Nonsymmetric NMF}
\label{sec:SNMF to NMF}
%

\revise{\subsection{Notations} \label{sec:notation}
We begin by introducing some notations. For the purpose of technical analysis, we may rewrite \eqref{eq:SNMF by reg} as an unconstrained optimization problem using indicator function,
\[
	\min_{\mU,\mV} \ f(\mU,\mV) = g(\mU,\mV) + \sigma_{+}(\mU) + \sigma_{+}(\mV),
\]
 with $ \sigma_{+}$ being the indicator function of nonnegative constraint defined as $\sigma_{+}(\mV) = \begin{cases}
                     0, & \mV \geq 0, \\
                     +\infty, & \mbox{otherewise}, 
                   \end{cases}$.
Upper boldface (such as $\mU$)  and lower boldface (such as $\vu$) respectively denote matrices and vectors in real Euclidean space.  $\mA\odot \mB$ represents the Hadamard product of two matrices. $\< \mA, \mB\> = \trace(\mA^\top \mB)$ represents the inner product of two matrices.   Throughout this paper, $k$ represents the iteration number only.  

\subsection{Penalized Nonsymmetric NMF is Equivalent to Symmetric NMF}}

Compared with \eqref{eq:SNMF}, at first glance, \eqref{eq:SNMF by reg} is slightly more complicated as it has one more variable. However, because of this new variable, $f(\mU,\mV)$ is now strongly convex with respect to either $\mU$ or $\mV$, though it is still nonconvex in terms of the joint variable $(\mU,\mV)$. Moreover, the two decision variables $\mU$ and $\mV$ in \eqref{eq:SNMF by reg} are well separated, as the case in the general nonsymmetric NMF. This observation suggests an interesting and useful fact that  \eqref{eq:SNMF by reg} can be solved by  tailored alternating-type  algorithms. On the other hand, a theoretical question raised in the penalized nonsymmetric form \eqref{eq:SNMF by reg} is whether we are  guaranteed $\mU = \mV$ and hence solving \eqref{eq:SNMF by reg} is equivalent to solving \eqref{eq:SNMF}. In this section, we provide an assuring answer to this question that solving \eqref{eq:SNMF by reg} (to a critical point) indeed gives a critical point solution of \eqref{eq:SNMF}. Note that problem \eqref{eq:SNMF} is nonconvex, and thus many local search algorithms can only be guaranteed to converge to its critical point rather than global minimizer.  

Before stating out the formal result, we first consider a simple case, as an intuitive example, where $f(u,v) = (1-uv)^2/2 + \lambda(u-v)^2/2$. Its derivative is $\partial_uf(u,v) = (uv -1)v + \lambda (u-v), \partial_vf(u,v) = (uv -1)u -  \lambda (u-v)$. Thus, any critical point of $f$ satisfies $(uv -1)v + \lambda (u-v)= 0$ and $(uv -1)u -  \lambda (u-v) = 0$, further indicating that $(u-v)(2\lambda +1 - uv ) = 0$. Therefore, for any critical point $(u,v)$ such that $|uv|<2\lambda +1$, it must satisfy $u = v$.  Although \eqref{eq:SNMF by reg} is more complicated as it also has nonnegative constraint, the following result establishes similar guarantee for \eqref{eq:SNMF by reg}.

\begin{thm} Let $(\mU^\star,\mV^\star)$ be any critical point of \eqref{eq:SNMF by reg} satisfying $\|\mU^\star\mV^{\star\T}\|_2 < 2 \lambda + \sigma_n(\mX)$, where $\|\cdot\|_2$ denotes the spectral norm and $\sigma_i(\cdot)$ denotes the $i$-th largest eigenvalue. Then $\mU^\star = \mV^\star$ and $\mU^\star$ is a critical point of \eqref{eq:SNMF}.
	\label{thm:U = V}\end{thm}

\begin{proof}[Proof of \Cref{thm:U = V}]
We first present the following useful result for any symmetric $\mA\in \R^{n\times n}$ and PSD matrix $\mB\in \R^{n\times n}$ \cite[Lemma 1]{wang1986trace},	
\e\label{eq:trace inequality for one PSD}
\sigma_n(\mA)\trace(\mB)\leq \trace\left(\mA\mB\right) \leq \sigma_1(\mA)\trace(\mB).
\ee
	
	We now prove \Cref{thm:U = V}. The subdifferential of $f$ is given as follows
	\e\begin{split} \label{eq:subdifferential U V}
		&\partial_{\mU}f(\mU,\mV) = (\mU\mV^\T - \mX)\mV + \lambda (\mU - \mV) + \partial \delta_+(\mU),\\
		&\partial_{\mV}f(\mU,\mV) = (\mU\mV^\T - \mX)^\T\mU - \lambda (\mU - \mV) + \partial \delta_+(\mV),
	\end{split}\ee
	where $\partial \delta_+(\mU) = \left\{\mG\in\R^{n\times r}:\mG \odot \mU = \vzero, \mG \leq \vzero \right\}$ when $\mU\geq \mzero$  and otherwise $\partial \delta_+(\mU) = \emptyset$. Since $(\mU^\star,\mV^\star)$ is a critical point of \eqref{eq:SNMF by reg}, it satisfies
	\begin{align}
	(\mU^\star\mV^{\star\T} - \mX)\mV^\star + \lambda (\mU^\star - \mV^\star) + \mG = \vzero,\label{eq:first order U}\\
	(\mU^\star\mV^{\star\T} - \mX)^\T\mU^\star - \lambda (\mU^\star - \mV^\star) + \mH = \vzero,\label{eq:first order V}
	\end{align}
	where $\mG\in  \partial \delta_+(\mU^\star)$ and $\mH\in  \partial \delta_+(\mV^\star)$. Subtracting \eqref{eq:first order V} from \eqref{eq:first order U}, we have
\e
\begin{split}
	(2\lambda \mId &+ \mX)(\mU^\star - \mV^\star)\\
& =  \mV^\star\mU^{\star\T}\mU^\star - \mU^\star\mV^{\star\T}\mV^\star - \mG + \mH.
\end{split}
\ee
	where we utilized the fact that $\mX$ is symmetric, i.e., $\mX = \mX^\T$. Taking the inner product of $\mU^\star - \mV^\star$ with both sides of the above equation gives
\e
 \begin{split}
	&\langle (\lambda \mId + \mX),(\mU^\star - \mV^\star)(\mU^\star - \mV^\star)^\T\rangle \\
&= \langle  \mV^\star\mU^{\star\T}\mU^\star - \mU^\star\mV^{\star\T}\mV^\star - \mG + \mH, \mU^\star - \mV^\star \rangle.
\end{split}
\label{eq: U V innder product}\ee
	
	In what follows, by choosing sufficiently large $\lambda$, we show that $(\mU^\star,\mV^\star)$ satisfying \eqref{eq: U V innder product} must satisfy $\mU^\star = \mV^\star$. To that end, we first provide the lower bound and the upper bound for the left-hand side and right-hand side of \eqref{eq: U V innder product}, respectively. Specifically,
\e
\begin{split}
	&\langle ((2\lambda \mId + \mX),(\mU^\star - \mV^\star)(\mU^\star - \mV^\star)^\T\rangle  \\
&\geq \sigma_n((2\lambda \mId + \mX)\|\mU^\star - \mV^\star\|_F^2\\
& = ((2\lambda + \sigma_n( \mX))\|\mU^\star - \mV^\star\|_F^2,
\label{eq:LHS}
\end{split}
\ee
	where the inequality follows from \eqref{eq:trace inequality for one PSD}. On the other hand,
	\e\begin{split}
		&\langle  \mV^\star\mU^{\star\T}\mU^\star - \mU^\star\mV^{\star\T}\mV^\star - \mG + \mH, \mU^\star - \mV^\star \rangle\\
        & \leq \langle  \mV^\star\mU^{\star\T}\mU^\star - \mU^\star\mV^{\star\T}\mV^\star, \mU^\star - \mV^\star \rangle\\
		&= \left\langle \frac{\mV^\star\mU^{\star\T} + \mU^\star\mV^{\star\T}}{2}, (\mU^\star - \mV^\star)(\mU^\star - \mV^\star)^\T \right\rangle\\
&\quad - \frac{1}{2}\left\| \mU^\star\mV^{\star\T}- \mV^\star\mU^{\star\T}\right\|_F^2\\
		&  \leq \left\langle \frac{\mV^\star\mU^{\star\T} + \mU^\star\mV^{\star\T}}{2}, (\mU^\star - \mV^\star)(\mU^\star - \mV^\star)^\T \right\rangle\\
		& \leq \sigma_1\left( \frac{\mV^\star\mU^{\star\T} + \mU^\star\mV^{\star\T}}{2} \right)\|\mU^\star - \mV^\star\|_F^2,
	\end{split}\label{eq:RHS}\ee
	where the last inequality utilizes \eqref{eq:trace inequality for one PSD} and the first inequality follows because $\mV^\star,\mU^\star\geq \vzero$ indicating that
	\[
	-\langle\mG , \mU^\star - \mV^\star\rangle \leq  0, \ \ \  \langle\mH , \mU^\star - \mV^\star\rangle \leq  0.
	\]
	
	Now plugging \eqref{eq:LHS} and \eqref{eq:RHS} back into \eqref{eq: U V innder product} and using the fact that $\sigma_1\big( \frac{\mV^\star\mU^{\star\T} + \mU^\star\mV^{\star\T}}{2} \big) \le \|\mU^\star\mV^{\star\T}\|_2$, we have
	\begin{align*}
	((2\lambda &+ \sigma_n(\mX)) \|\mU^\star - \mV^\star\|_F^2 \\
&\leq \sigma_1\left( \frac{\mV^\star\mU^{\star\T} + \mU^\star\mV^{\star\T}}{2} \right)\|\mU^\star - \mV^\star\|_F^2 \\
&\leq \|\mU^\star\mV^{\star\T}\|_2\|\mU^\star - \mV^\star\|_F^2,
	\end{align*}
	which implies that if we choose $2\lambda> \|\mU^\star\mV^{\star\T}\|_2 - \sigma_n(\mX)$, then $\mU^\star = \mV^\star$ must hold. Plugging it into \eqref{eq:subdifferential U V} gives
	\[
	\mzero\in (\mU^\star(\mU^\star)^\T - \mX)\mU^\star + \partial \delta_+(\mU^\star),
	\]
	which implies $\mU^\star$ is a critical point of \eqref{eq:SNMF}.
\end{proof}

Several remarks on \Cref{thm:U = V} are made as follows. First, the strategy of solving the symmetric NMF by targeting on a nonsymmetric one  can be naturally extended to multiple variables, such as symmetric tensor factorization. Investigation along this line is of interest and is the subject of future work.  Also, note that for any $\lambda>0$, \Cref{thm:U = V} ensures a certain region (whose size depends on $\lambda$) in which each critical point of \eqref{eq:SNMF by reg} has identical factors and also returns a solution for the original symmetric NMF \eqref{eq:SNMF}. This further suggests the opportunity of choosing an appropriate $\lambda$ such that the corresponding region (i.e., all $(\mU,\mV)$ such that $\|\mU\mV^{\T}\| < 2 \lambda + \sigma_n(\mX)$) contains all the possible points that the algorithms will converge to. 
The rest is to argue that for any local search algorithms when used to solve \eqref{eq:SNMF by reg}, if it decreases the objective function, then the iterates are bounded.

\begin{lem} For any local search algorithm solving \eqref{eq:SNMF by reg} with initialization $\mV_0=\mU_0, \mU_0\geq0$, suppose it sequentially decreases the objective value.  Then, for any $k\geq 0$,  the iterate $(\mU_k,\mV_k)$ generated by this algorithm satisfies
\e
\begin{split}
		\|\mU_k\|_F^2+\|\mV_k\|_F^2& \leq \left(\frac{1}{\lambda}+2\sqrt{r}\right)\|\mX-\mU_0\mU_0^\T\|_F^2 \\
&\quad+2\sqrt{r} \|\mX\|_F:=B_0,\\
		\|\mU_k\mV_k^\T\|_2&\leq\|\mX-\mU_0\mV_0^\T\|_F+ \|\mX\|_2.
	\end{split}
\label{eqn:bound}\ee

	
	\label{lem:bound:iterate}
\end{lem}

\begin{proof}[Proof of \Cref{lem:bound:iterate}]
	By the assumption that the algorithm decreases the objective function, we have
	\begin{align*}
	\frac{1}{2}\left\|\mX - \mU_k \mV_k^{\T }\right \|_F^2 + \frac{\lambda}{2} \left\|\mU_k-\mV_k \right\|_F^2
	\leq \frac{1}{2}\left \|\mX - \mU_0\mU_0^{\T} \right\|_F^2
	\end{align*}
	which further implies that
\e
  \left\|\mX-\mU_k \mV_k^{\T}\right\|_F \leq \left\| \mX-\mU_0\mU_0^\T\right\|_F, 
\label{eq:proof-Lemma1-1}\ee
and 
\e
	\begin{split}
	&\frac{\lambda}{2} \left(\|\mU_k\|_F^2+\|\mV_k\|_F^2  -2|\lg\mU_k\mV_k^\T,\mId_r\rg|\right) \\
&= \frac{\lambda}{2} \left\|\mU_k-\mV_k \right\|_F^2 \leq \frac{1}{2}\left \|\mX - \mU_0\mU_0^{\T} \right\|_F^2.
	\end{split}
\label{eq:proof-Lemma1-2}\ee

It follows from \eqref{eq:proof-Lemma1-1} that
	\begin{align*}
\|\mU_k\mV_k^\T\|_2 - \|\mX\|_2 &\le\|\mU_k\mV_k^\T - \mX\|_2 \le 
 \|\mU_k\mV_k^\T - \mX\|_F \\& \le \|\mX-\mU_0\mV_0^\T\|_F.
	\end{align*}
Also, the inequality in \eqref{eq:proof-Lemma1-2} leads to
	\begin{align*}
	\|\mU_k\|_F^2 &+\|\mV_k\|_F^2
	\leq \frac{1}{\lambda}\|\mX-\mU_0\mU_0^\T\|_F^2+2\|\mU_k\mV_k^\T\|_F\|\mId_r\|_F\nonumber\\
	&=
	\frac{1}{\lambda}\|\mX-\mU_0\mU_0^\T\|_F^2+2\sqrt{r}\|\mU_k\mV_k^\T\|_F\nonumber\\
	&\leq
	\left(\frac{1}{\lambda}+2\sqrt{r}\right)\|\mX-\mU_0\mU_0^\T\|_F^2+2\sqrt{r} \|\mX\|_F \\
&=:B_0
	\end{align*}	
\end{proof}

%
%

There are two interesting facts regarding the iterates that can be interpreted from \eqref{eqn:bound}. The first equation of \eqref{eqn:bound} implies that both $\mU_k$ and $\mV_k$ are bounded and the upper bound decays when the  $\lambda$ increases. Specifically, as long as $\lambda$ is not too close to zero, then the right-hand side in \eqref{eqn:bound} gives a meaningful bound which will be used for the convergence analysis of local search algorithms in next section.
In terms of $\mU_k\mV_k^\T$, the second equation of \eqref{eqn:bound} indicates that it is indeed upper bounded by a quantity that is independent of $\lambda$. This suggests a key result that if  the iterative algorithm is convergent and the iterates $(\mU_k,\mV_k)$  converge to a critical point $(\mU^\star,\mV^\star)$, then  $\mU^\star\mV^{\star T}$ is also bounded, irrespectively the value of $\lambda$. This together with \Cref{thm:U = V} ensures that many local search algorithms can be utilized to find a critical point of \eqref{eq:SNMF} by targeting \eqref{eq:SNMF by reg} with a properly chosen  large $\lambda$.

\begin{thm}
	In \eqref{eq:SNMF by reg}, set
\e
\lambda > \frac{1}{2} \left( \|\mX\|_2 + \left\|\mX-\mU_0\mU_0^\T\right\|_F - \sigma_n(\mX)  \right).
\label{eq:lambda-lower-bound}\ee
 For any local search algorithm solving \eqref{eq:SNMF by reg}  with initialization $\mV_0=\mU_0$, if it sequentially decreases the objective function, is convergent, and  converges to a critical point $(\mU^\star,\mV^\star)$ of \eqref{eq:SNMF by reg}, then we have $\mU^\star = \mV^\star$ and  $\mU^\star$ is also a critical point of \eqref{eq:SNMF}.
	\label{thm:critical point U = V}
\end{thm}

\begin{proof}[Proof of \Cref{thm:critical point U = V}]
Since the assumptions of \Cref{lem:bound:iterate} are satisfied, it follows from \eqref{eqn:bound} that
\begin{align*}
\|\mU^\star\mV^{\star\T}\|_2&\leq\|\mX-\mU_0\mV_0^\T\|_F+\|\mX\|_2\\
&< 2 \lambda + \sigma_n(\mX),
\end{align*}
where the second line utilizes \eqref{eq:lambda-lower-bound}. We complete the proof by invoking \Cref{thm:U = V}.
\end{proof}

\remark{
\Cref{thm:critical point U = V} indicates that instead of directly solving the symmetric NMF \eqref{eq:SNMF}, one can turn to solve \eqref{eq:SNMF by reg} with a properly chosen penalty parameter $\lambda$. The latter has  similar form to the general nonsymmetric NMF \eqref{eq:NMF}  obeying  splitting property, which paves the way for designing efficient alternating-type algorithms.}

\section{Fast Algorithms for Symmetric NMF}
\label{sec:fast algorithms}

In the last section, we have shown that the symmetric NMF \eqref{eq:SNMF} can be transformed to problem \eqref{eq:SNMF by reg} which admits splitting property, enabling us to design efficient alternating-type algorithms to solve the original symmetric NMF. In this section, we exploit the splitting property and design fast algorithms for solving problem \eqref{eq:SNMF by reg} by adopting ANLS, HALS, and accelerated HALS. Moreover, we provide strong  convergence guarantees that the sequence of iterates generated by our algorithms is convergent and converges to a critical point of the original symmetric NMF \eqref{eq:SNMF}. This is obtained by exploiting \Cref{thm:critical point U = V} and  the property  that the objective function $f$ in \eqref{eq:SNMF by reg} is strongly convex with respect to $\mU$ (or $\mV$) when the other variable $\mV$ (or $\mU$) is fixed.

\subsection{ANLS-type Method for Symmetric NMF}
\begin{algorithm}[htb]
	\caption{SymANLS}
	\label{alg:ANLS}
	{\bf Initialization:}  $k=1$ and  $\mU_0 = \mV_0$.
	
	\begin{algorithmic}[1]
		\WHILE{stop criterion not meet}
		\STATE $
		\mU_{k} = \arg\min\limits_{\mU\geq 0}  \frac{1}{2}\|\mX - \mU\mV_{k-1}^\T\|_F^2   + \frac{\lambda}{2}\|\mU - \mV_{k-1}\|_F^2$;
		\STATE
		$\mV_{k} = \arg\min\limits_{\mV\geq 0}  \frac{1}{2}\|\mX - \mU_{k}\mV^\T\|_F^2 + \frac{\lambda}{2}\|\mU_k - \mV\|_F^2 $;
		\STATE
		$k = k+1$.
		\ENDWHILE
	\end{algorithmic}
	{\bf Output:} factorization $(\mU_k, \mV_k)$.
\end{algorithm}

ANLS is an alternating-type algorithm customized for nonsymmetric NMF \eqref{eq:NMF} and its main idea is  to keep one factor fixed and update another one  via solving a nonnegative least squares.  We use a similar idea for solving \eqref{eq:SNMF by reg} and refer to the corresponding algorithm as SymANLS; see \Cref{alg:ANLS}. Specifically,  at the $k$-th iteration, SymANLS first updates $\mU_k$ by
\e
\mU_k =\argmin_{\mU\in\R^{n\times r}, \mU\geq 0} \|\mX - \mU\mV_{k-1}^\T\|_F^2   + \frac{\lambda}{2}\|\mU - \mV_{k-1}\|_F^2.
\label{eq:Uk ANLS}\ee
$\mV_k$ is then updated in a similar way.  For solving the subproblem~\eqref{eq:Uk ANLS}, we first note that there exists a unique minimizer (i.e., $\mU_k$) for \eqref{eq:Uk ANLS} as it involves a strongly objective function as well as a convex constraint. Unlike least squares, however, in general there is no closed-form solution for \eqref{eq:Uk ANLS} (unless $r =1$) due to the nonnegative constraint. Fortunately, there exist many feasible methods to solve the nonnegative least squares, such as projected gradient descent, active set method and projected Newton's method. Among these methods,  a  block principal pivoting method is remarkably efficient for tackling the subproblem \eqref{eq:Uk ANLS} (and also the one for updating $\mV$) \cite{kim2008toward}.

\subsection{HALS-type Method for Symmetric NMF}
As we stated before,  due to the nonnegative constraint,  there is no closed-from solution for \eqref{eq:Uk ANLS}, although one may utilize some efficient algorithms for solving it. However, there does exist a closed-form solution when $r =1$. HALS \cite{cichocki2007hierarchical} exploits this observation by splitting the pair of variables $(\mU,\mV)$ into  columns $(\vu_1,\cdots,\vu_r,\vv_1,\cdots,\vv_r)$ and then optimizing over \emph{column by column}. We borrow this idea for solving \eqref{eq:SNMF by reg}. Specifically, rewrite $\mU\mV^\T = \vu_i\vv_i^\T + \sum_{j\neq i}\vu_j\vv_j^\T$ and denote
\[
\overline \mX = \mX - \sum_{j\neq i}\vu_j\vv_j^\T
\]
the factorization residual $\mX-\mU\mV^\T$ excluding the contribution of $\vu_i\vv_i^\T$. Now if we minimize the objective function $f$ in \eqref{eq:SNMF by reg} only with respect to $\vu_i$, then it is equivalent to
\begin{equation}
\label{eq:HALS:update}
\begin{aligned}
   \vu_i^\natural &= \argmin_{\vu_i\in\R^{n}} \frac{1}{2}\|\overline \mX - \vu_i\vv_i^\T\|_F^2 + \frac{\lambda}{2}\|\vu_i - \vv_i\|_2^2 \\
    &= \max\left(\frac{( \overline \mX +\lambda \mId)\vv_i}{\|\vv_i\|_2^2+\lambda },0\right).
\end{aligned}
\end{equation}
Similar closed-form solution also holds when optimizing in terms of $\vv_i$. The leads to our second algorithm, namely SymHALS (depicted in \Cref{alg:HALS}) which is an alternating-type minimization algorithm that at each time minimizes \eqref{eq:SNMF by reg} only with respect to one column in $\mU$ or $\mV$.

\begin{algorithm}[htb]
	\caption{SymHALS}
	\label{alg:HALS}
	{\bf Initialization:}  $k=1$ and  $\mU_0 = \mV_0$.
	
	\begin{algorithmic}[1]
		\STATE precompute residual $\overline \mX_k = \mX - \mU_{k-1}\mV_{k-1}^\T$.
		\WHILE{stop criterion not meet}
		\FOR{$i=1:r$}
		\STATE $\overline \mX_k \leftarrow \overline \mX_k + \vu_{i,k-1}(\vv_{i,k-1})^\T$
		\STATE $\vu_{i,k}   = \max\left(\frac{( \overline \mX_k +\lambda \mId)\vv_{i,k-1}}{\|\vv_{i,k-1}\|_2^2+\lambda },0\right)$;
		\STATE Update residual $\overline \mX_k \leftarrow \overline \mX_k - \vu_{i,k}(\vv_{i,k-1})^\T$.
		\ENDFOR
		\STATE Update $\mV_{k}$ using the same steps.
		\STATE$\overline \mX_{k+1}  = \overline \mX_k$, $k = k+1$.
		\ENDWHILE
	\end{algorithmic}
	{\bf Output:} factorization $(\mU_k, \mV_k)$.
\end{algorithm}

\subsection{Accelerated HALS-type Method for Symmetric NMF} \label{sec:A-SymHALS}
Compared with ANLS, HALS may need more iterations to converge since in each step it only updates one column. One effective approach \cite{gillis2012accelerated} to accelerate HALS is by updating one block variable (say $\mU$) several times before processing another block variable (say $\mV$), i.e., cyclically updating $\vu_1,\cdots, \vu_r$ by \eqref{eq:HALS:update} multiple times  before updating $\mV$. We adopt this strategy and denote the corresponding algorithm by A-SymHALS, which is depicted in \Cref{alg:Accelerated HALS}.

\begin{algorithm}[htb]
	\caption{A-SymHALS}
	\label{alg:Accelerated HALS}
	{\bf Initialization:}  $k=1$ and  $\mU_0 = \mV_0$, and inner iteration number $L$.
	
	\begin{algorithmic}[1]
		\WHILE{stop criterion not meet}
		\STATE Inner initialization: $\mU_{k}^0 = \mU_{k}$
		\FOR{$j = 1:L$}
		\STATE Update $\mU_{k}^j$ by performing steps 3-7 in \Cref{alg:HALS} with $\mV_k$ and $\mU_{k}^{j-1}$.
		\ENDFOR
		\STATE Set $\mU_{k+1} = \mU_{k}^L$.
		\STATE Repeat the above process (steps 2-6) to update $\mV_{k+1}$.
		\STATE $k = k+1$.
		\ENDWHILE
	\end{algorithmic}
	{\bf Output:} factorization $(\mU_k, \mV_k)$.
\end{algorithm}

\remark{ SymHALS updates each column of $\mU$ and $\mV$ one time during each iteration, while A-SymHALS updates each column multiple  times by refining previous solutions. \revise{Thus, on the one hand,  A-SymHALS has higher computational complexity than SymHALS in each iteration, but less than SymANLS since the latter requires to solve relatively computationally expensive nonnegative least squares in each iteration. On the other hand, A-SymHALS is supposed to converge faster than SymHALS since the former decreases more function value in each iteration, while SymANLS can decrease the most function value among them in each iteration. Therefore, A-SymHALS can be viewed as an effective approach to balance the trade-off between convergence speed and computational complexity of each iteration within SymHALS and SymANLS. We will compare these three algorithms in \Cref{sec:experiments}. But before this, in the next Section we provide convergence analysis for the three algorithms and show that all of them will converge to a critical point of the original symmetric NMF problem \eqref{eq:SNMF}. 
  } 

 }

\subsection{Convergence Results}

By exploiting the strong convexity of the objective function in \eqref{eq:SNMF by reg} when restricted to block $\mU$ (or $\mV$)  and the guarantee of \Cref{thm:critical point U = V}, we establish the convergence result of the three proposed alternating algorithms.

\begin{thm}[Convergence  of the proposed algorithms to a critical point of the symmetric NMF \eqref{eq:SNMF}] In \eqref{eq:SNMF by reg}, set
	\e
	\lambda > \frac{1}{2} \left( \|\mX\|_2 + \left\|\mX-\mU_0\mU_0^\T\right\|_F - \sigma_n(\mX)  \right)=:\overline\lambda.
	\label{eq:lambda-lower-bound 2}\ee
	Suppose \Cref{alg:ANLS}, \Cref{alg:HALS}, and \Cref{alg:Accelerated HALS} are initialized  with $\mV_0=\mU_0$. Let $\{(\mU_k,\mV_k)\}_{k\geq0}$ be the sequence of iterates generated by any of the three algorithms.  Then,
	\[
	\lim_{k\rightarrow \infty} (\mU_k,\mV_k) = (\mU^\star,\mV^\star),
	\]
where the limit point $(\mU^\star,\mV^\star)$ satisfies $\mU^\star = \mV^\star $
	and $\mU^\star$	is a critical point of \eqref{eq:SNMF}. Furthermore, the convergence rate is at least sublinear.
	\label{thm:convergence}
\end{thm}

The proof of \Cref{thm:convergence} consists of showing the convergence requirement of algorithms in \Cref{thm:critical point U = V}; i.e., the  decent and convergence properties. We defer the detailed proof to  Section \ref{Appendix}.

\remark{\label{remark:convergence} \revise{First note that the algorithm can be proved to converge to $(\mU^\star,\mV^\star)$ for any positive $\lambda$ with the argument in Section \ref{Appendix}, but we need $\lambda$ to be relatively large as in \eqref{eq:lambda-lower-bound 2} to ensure $\mU^\star = \mV^\star$.}  We emphasize that the specific structure within \eqref{eq:SNMF by reg} enables \Cref{thm:convergence} to get rid of both the assumption on the boundedness of the iterates $\{(\mU_{k},\mV_k)\}_{k\geq0}$ and  the requirement of an additional proximal term, which are usually required for convergence analysis though are not necessary in practice \cite{attouch2009convergence,attouch2010proximal}.  For example, the previous work \cite{huang2016flexible} provides  convergence guarantee for the standard ANLS when used to solve the general nonsymmetric NMF \eqref{eq:NMF} by adding an additional proximal term as well as  an additional constraint to force the factors bounded. To establish the convergence for the standard HALS  for solving \eqref{eq:NMF} \cite{cichocki2009fast,gillis2012accelerated}, one needs the assumption that every column in $(\mU_k,\mV_k)$ is away from zero through all iterations. Though such an assumption can be satisfied by explicitly imposing additional constraints, it leads to a slightly different problem. By contrast,  our convergence result when applied to SymHALS and A-SymHALS overcomes this issue because of the additional penalty term in \eqref{eq:SNMF by reg}.} 

\remark{ \label{remark: proof of A-SymHALS} The convergence of SymANLS and SymHALS can be established by following the  Kurdyka-Lojasiewicz (KL) convergence analysis framework \cite{attouch2010proximal,bolte2014proximal} since the additional penalty term in \eqref{eq:SNMF by reg} can be used for establishing the so-called sufficient decrease property of the two algorithms.  However, the multiple  update scheme of    A-SymHALS makes its convergence analysis more complicated than the previous two algorithms. As a result, one cannot directly apply this framework for A-SymHALS.  To overcome this technical difficulty, we generalize the KL analysis framework such that it becomes compatible with the multiple update scheme used in A-SymHALS, which is of independent interest; see Section \ref{Appendix} for detailed analysis. }

\revise{ \subsection{An Adaptive Updating Formula for $\lambda$} 
	\Cref{thm:convergence} guarantees convergence of the three algorithms  for any relatively large penalty parameter $\lambda$. Though $\lambda$ is fixed in Algorithms 1-3 for simplicity, it can be updated through the entire process. In this subsection, we  provide an adaptive strategy for updating the parameter $\lambda$  along with the iterations of our algorithms. Towards that end, first note that $
	\|\mU_k \pm \mV_k \|_F^2 \geq 0$ always holds, which gives
	\[
	\|\mU_k\|_F^2 + \|\mV_k\|_F^2 \geq 2 \left|\< \mU_k, \mV_k \> \right|. 
	\]
	This further implies
	\e\label{eq:motivate}
	\frac{ \|\mU_k\|_F^2 + \|\mV_k\|_F^2}{2 \left| \< \mU_k, \mV_k \> \right|} \geq 1.
	\ee
	The relation in \eqref{eq:motivate} motivates us to use the following \emph{adaptive} strategy  for updating $\lambda$:
	\e\label{eq:update lambda}
	\boxed{ \quad 
		\lambda_{k+1} = \lambda_{k} \cdot \frac{ \|\mU_k\|_F^2 + \|\mV_k\|_F^2}{ 2 \left| \< \mU_k, \mV_k \> \right|}, \quad k \geq0,
		\quad 	}
	\ee
	where the initial regularization parameter $\lambda_0>0$ can be selected as a very small number, e.g., $\lambda_0= 10^{-5}$. Note that $\mU_k $ and $\mV_k$  will not tend to be orthogonal as we are minimizing $\|\mU-\mV\|_F^2$. Therefore, the formula \eqref{eq:update lambda} is well defined as the denominator is bounded alway from zero as the iteration proceeds.  
	
To understand this adaptive strategy, we note that by \eqref{eq:motivate}, the parameter $\lambda_k$ keeps increasing at each iteration until 
	\[\frac{ \|\mU_k\|_F^2 + \|\mV_k\|_F^2}{ 2 \left| \< \mU_k, \mV_k \> \right| } = 1.\]
The above happens only when\footnote{ \revise{
It can also happen when $\mU_k = -\mV_k$, but since the algorithm minimizes $\|\mU - \mV\|_F^2$, this case in practice is likely impossible.}} $\mU_k = \mV_k$, which indeed is the goal we want to achieve. On the one hand, the  two factors  can reach consensus without requiring $\lambda_k$ to go to infinitely large, as guaranteed by \Cref{thm:convergence}. On the other hand,  we note that the algorithm may already converge and $\mU^\star = \mV^\star$ even with  $\lambda_k\le \overline\lambda$, where $\overline\lambda$ is the lower bound in \eqref{eq:lambda-lower-bound 2}. There is no contradiction to \Cref{thm:convergence} as \eqref{eq:lambda-lower-bound 2} is a sufficient but not necessary condition to guarantee convergence and consensus. In fact, this is one of the advantages of using this adaptive strategy since a smaller penalty term can allow the algorithms focus more on the data fidelity term.  We refer to \Cref{fig:lambda} in next Section for an illustration of the practical performance of the adaptive strategy \eqref{eq:update lambda}.   
	
It is  worth mentioning that our idea of updating the penalty parameter $\lambda$  has the potential to be used widely in other Augmented Lagrangian-based algorithms and splitting algorithms since these type of methods usually use squared Euclidean distance to penalize  violation of the constraints. }

\section{Numerical Experiments}
\label{sec:experiments}
In this section, we conduct experiments on both synthetic data and real image clustering to illustrate the performance of our proposed algorithms and compare it with other state-of-the-art ones, in terms of both  convergence property and  clustering accuracy.

For comparison in terms of solving the original symmetric NMF  \eqref{eq:SNMF}, we define
\[
E^k = \frac{\|\mX-\mU_k(\mU_k)^\T\|_F^2}{\|\mX\|_F^2}
\]
as the normalized fitting error at the $k$-th iteration.

Besides SymANLS, SymHALS and A-SymHALS \revise{for which we set the inner iteration number $L = 2$ in \Cref{alg:Accelerated HALS}},  we list several state-of-the-art algorithms to compare: 1) ADMM  \cite{lu2017nonconvex}, which solves an equality constrained nonsymmetric NMF using a primal dual method;  2) \revise{truncated SVD (\textbf{tSVD})} \cite{huang2013non}, which utilizes SVD to iteratively approximate the solution to the original symmetric NMF  \eqref{eq:SNMF}; 3) \revise{beta Symmetric Nonnegative Matrix Factorization (\textbf{beta-SNMF})} \cite{he2011symmetric}, which is an accelerated version of the well-known multiplicative update algorithm  in NMF literature;  4) PGD \cite{lin2007projected}.  

\subsection{Experiments with Synthetic Data}
We randomly generate a matrix $\mU\in\R^{n\times r}$ with each entry independently following a standard Gaussian distribution. To enforce nonnegativity, we then take absolute value on each entry of $\mU$ to get $\mU^\star$.  Data matrix $\mX$ is constructed as $\mX = \mU^\star \mU^{\star\T} + \sigma|\mN|$, where $\mN$ represents the noise and $\sigma$ is the noise level. 
Unless explicitly specified,  each entry of the noisy matrix $\mN$ follows an \emph{i.i.d.} standard  Gaussian distribution.  We initialize all  algorithms with  $\mU_0$, whose entries are \emph{i.i.d.} uniformly distributed  between 0 and 1.

\revise{
	
	\begin{figure}[!t]
	\centering
	
	\begin{minipage}{0.49\linewidth}
		\includegraphics[width=1\textwidth]{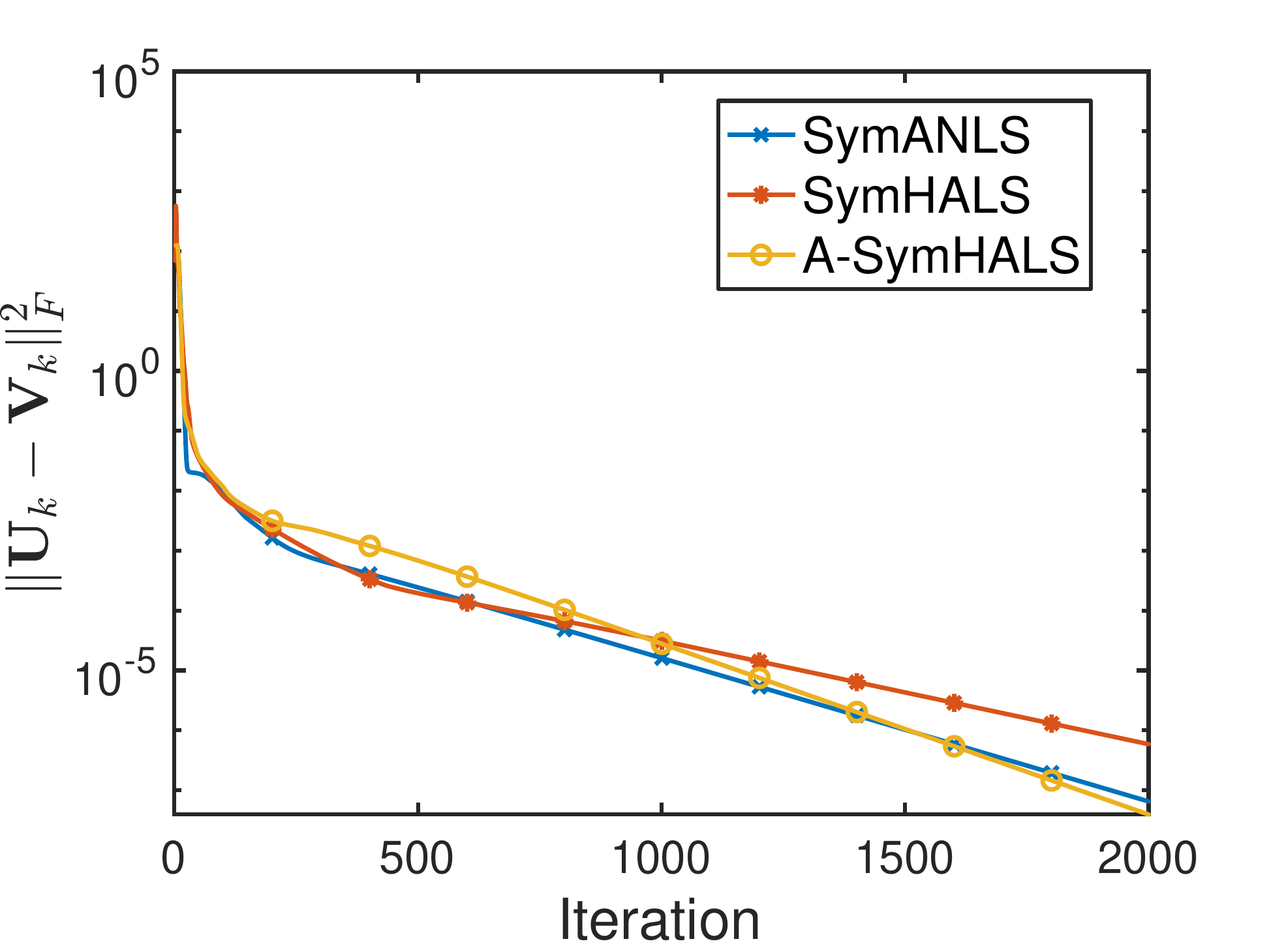}\\
		\centering{(a1)}
	\end{minipage}
	\hfill
	\begin{minipage}{0.49\linewidth}
		\includegraphics[width=1\textwidth]{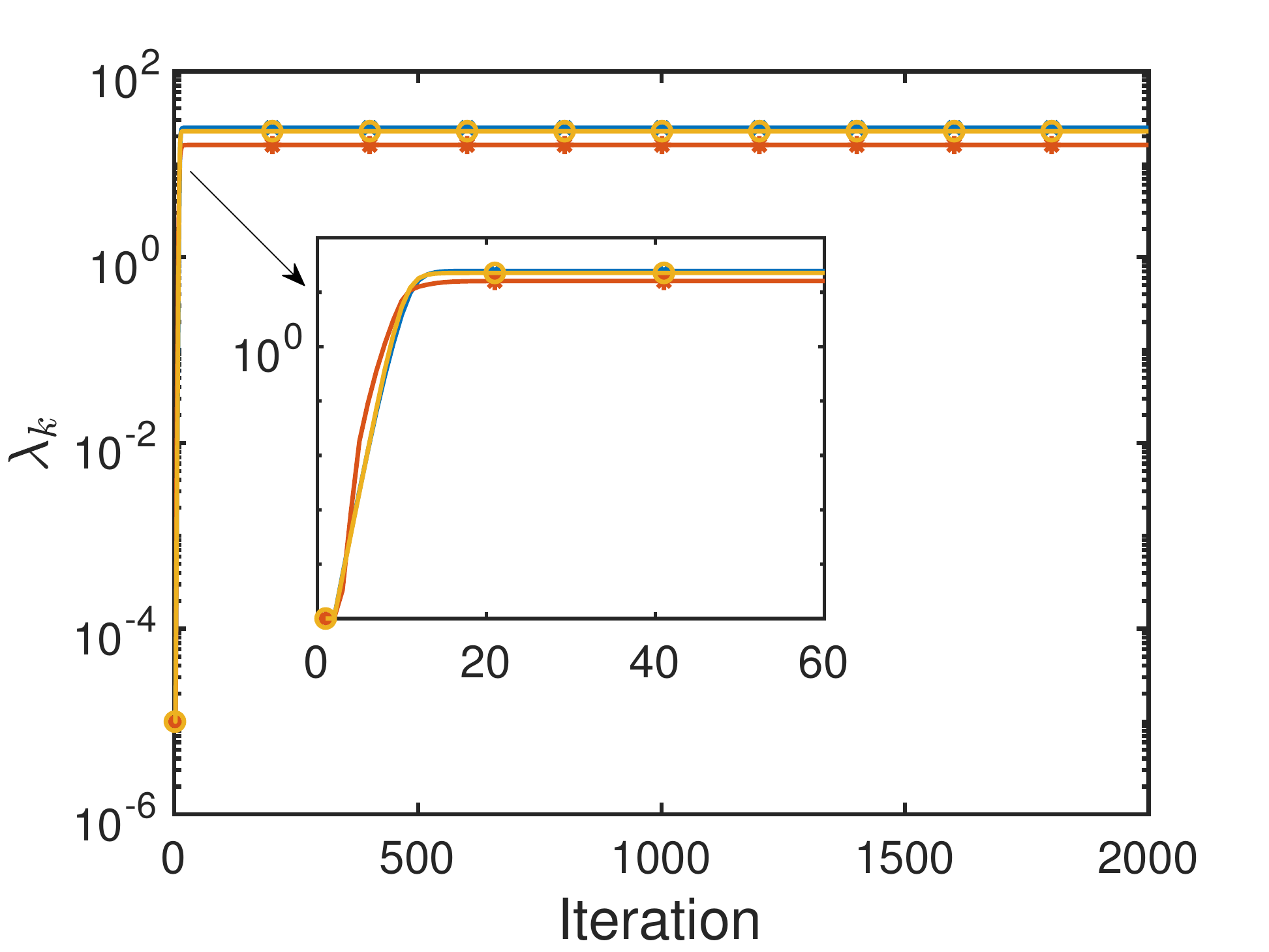}\\
		\centering{(a2)}
	\end{minipage}
	\caption{ \revise{(a1): The penalty term $\|\mU_k - \mV_k\|_F^2$ versus iteration number. (a2): The penalty parameter $\lambda_k$ versus iteration number.  Here, $\lambda_0 = 10^{-5}, n = 50, r = 5, \sigma = 0$.}}  \label{fig:lambda}
\end{figure}

We first verify the adaptive formula \eqref{eq:update lambda} for updating the parameter $\lambda$. \Cref{fig:lambda} displays the penalty term $\|\mU_{k}-\mV_{k}\|_F^2$ and the penalty parameter $\lambda_k$ updated using \eqref{eq:update lambda} versus iteration count $k$. One can observe that the term $\|\mU_{k}-\mV_{k}\|_F^2$  converges to 0 and $\lambda_{k}$ converges to a finite value for our proposed algorithms. In the sequel, we utilize \eqref{eq:update lambda} to update the parameter $\lambda$ with $\lambda_0 = 10^{-5}$ for all experiments.}


\begin{figure}[!t]
	\centering
	
	\begin{minipage}{0.49\linewidth}
		\includegraphics[width=1\textwidth]{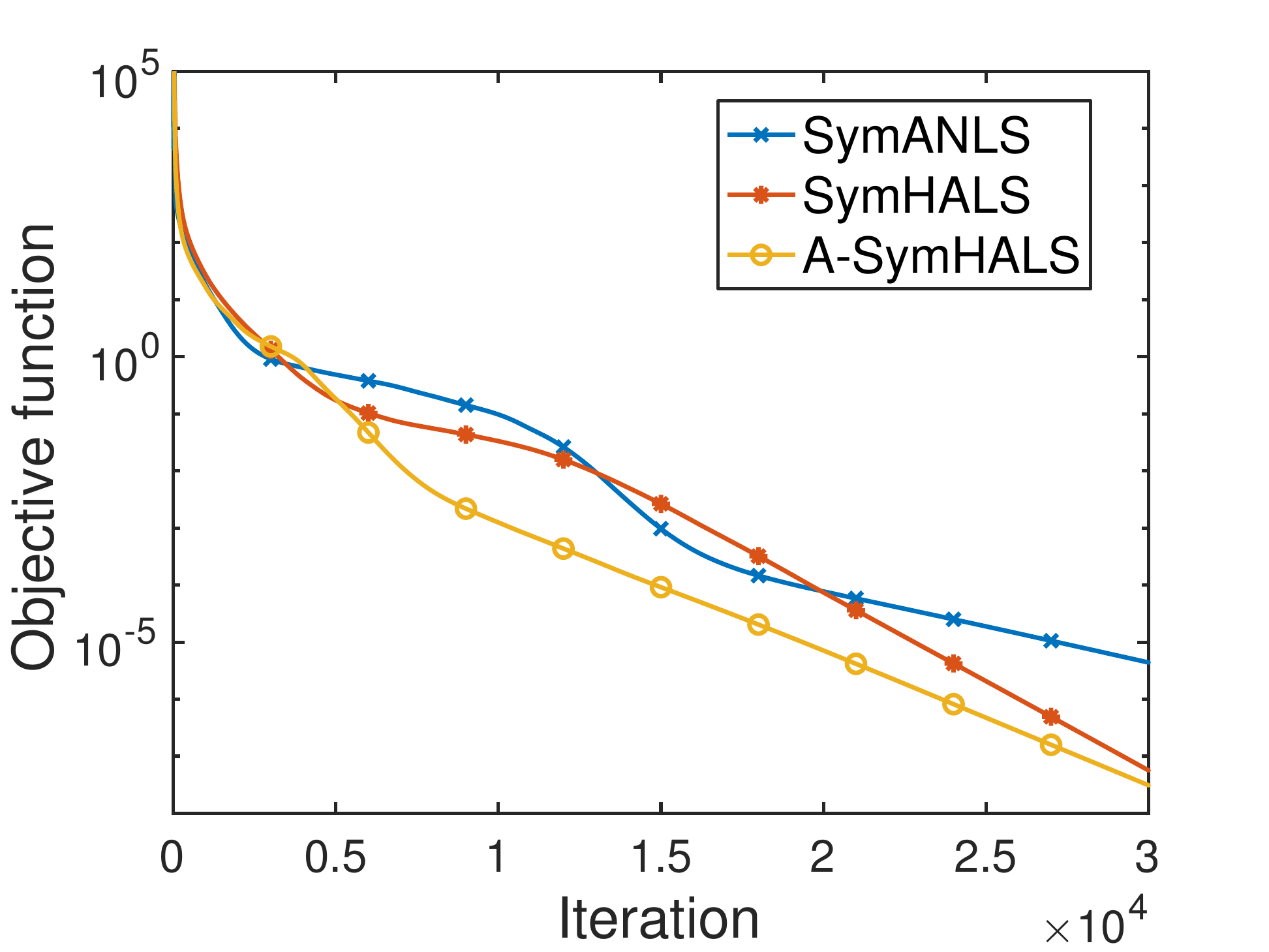}\\
		\centering{(a1) $\sigma = 0$}
	\end{minipage}
	\hfill
	\begin{minipage}{0.49\linewidth}
		\includegraphics[width=1\textwidth]{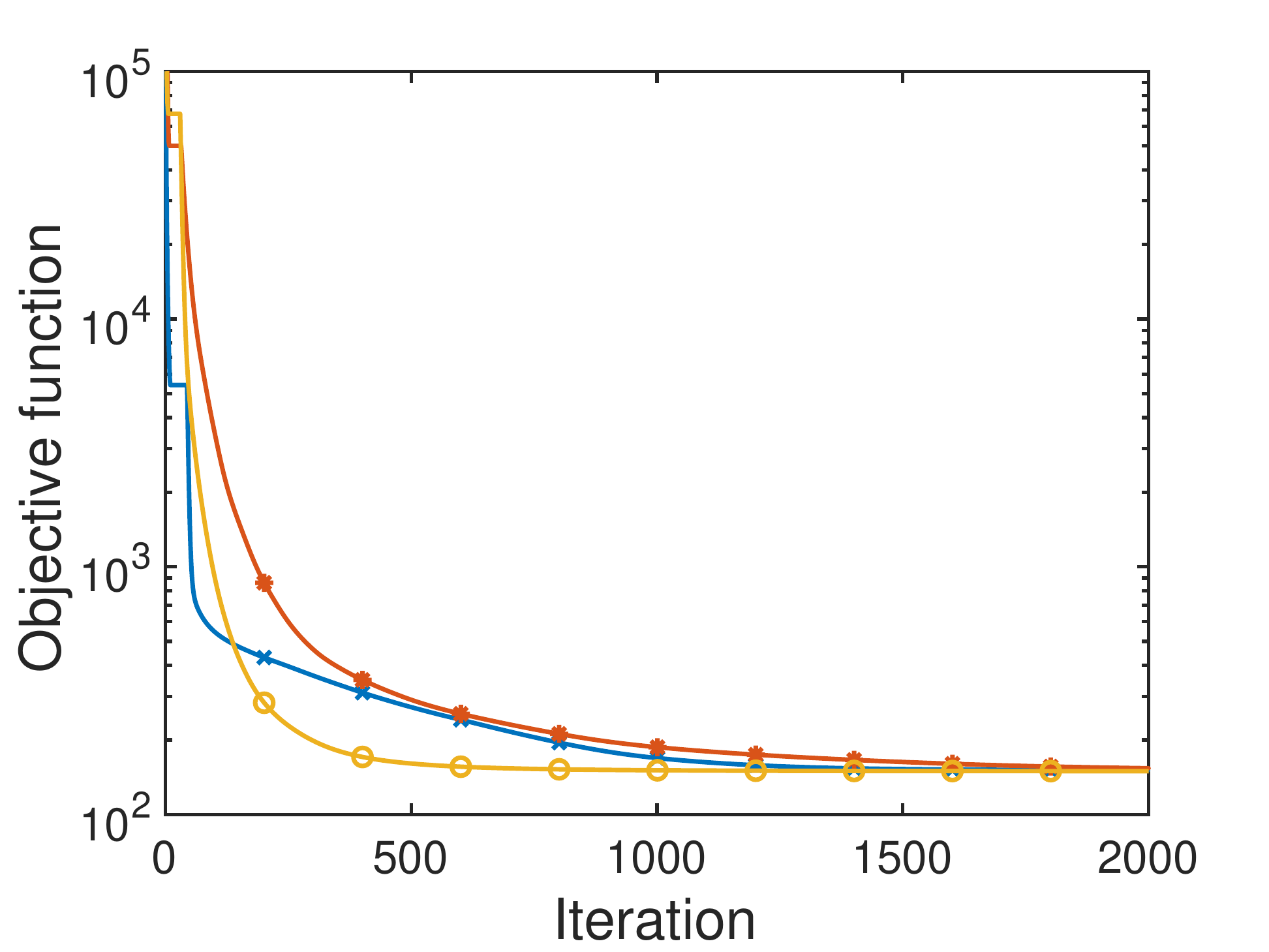}\\
		\centering{(a2) $\sigma = 0.1$.}
	\end{minipage}
	\caption{Convergence of the proposed algorithms on synthetic data with different noise level $\sigma$, where $n = 300, r = 20$. }  \label{fig:convergence f}
\end{figure}

\begin{figure}[!t]
	\centering
	
	\begin{minipage}{0.49\linewidth}
		\includegraphics[width=1\textwidth]{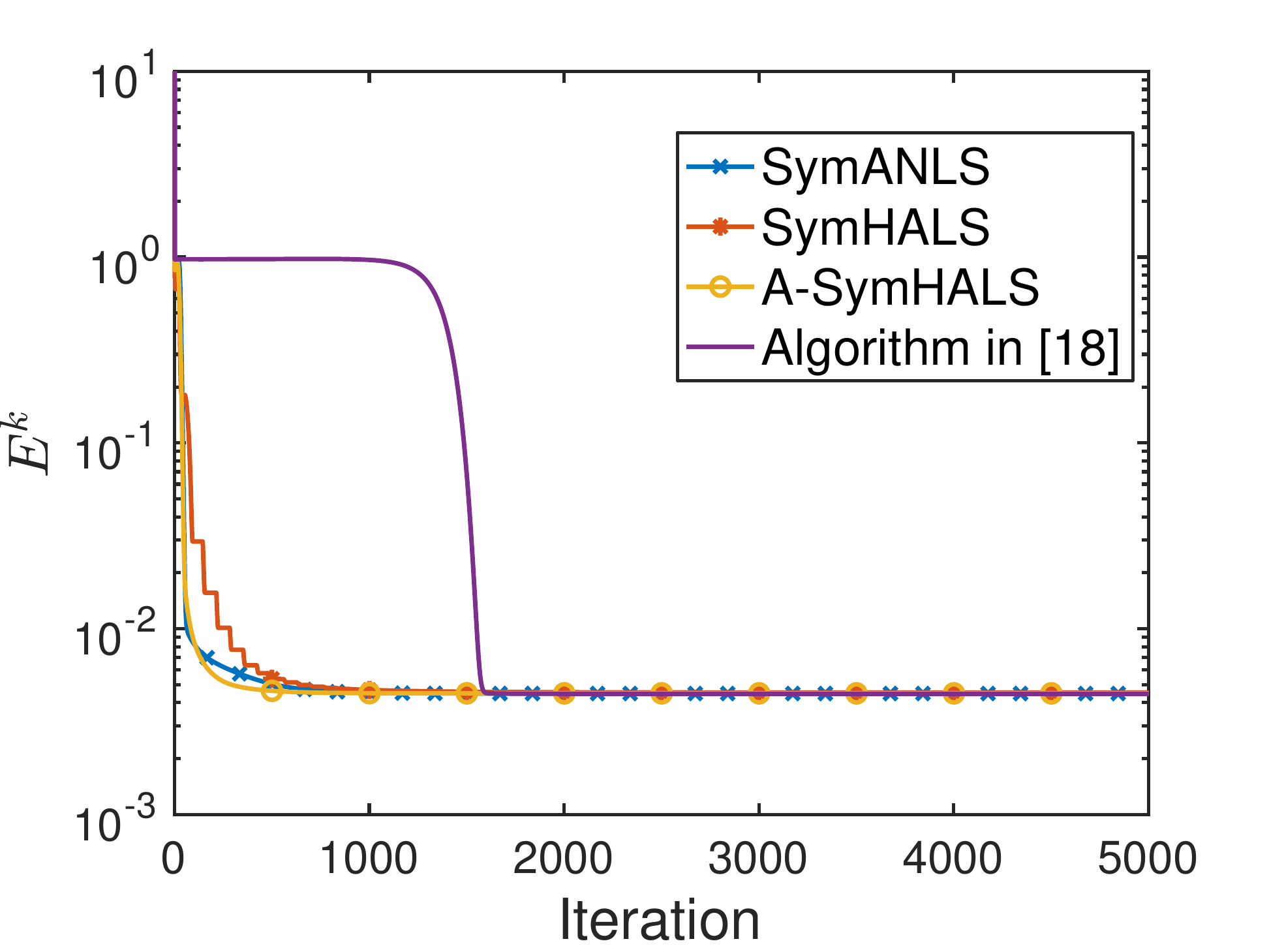}\\
		\centering{(a1) $r= 20$}
	\end{minipage}
	\hfill
	\begin{minipage}{0.49\linewidth}
		\includegraphics[width=1\textwidth]{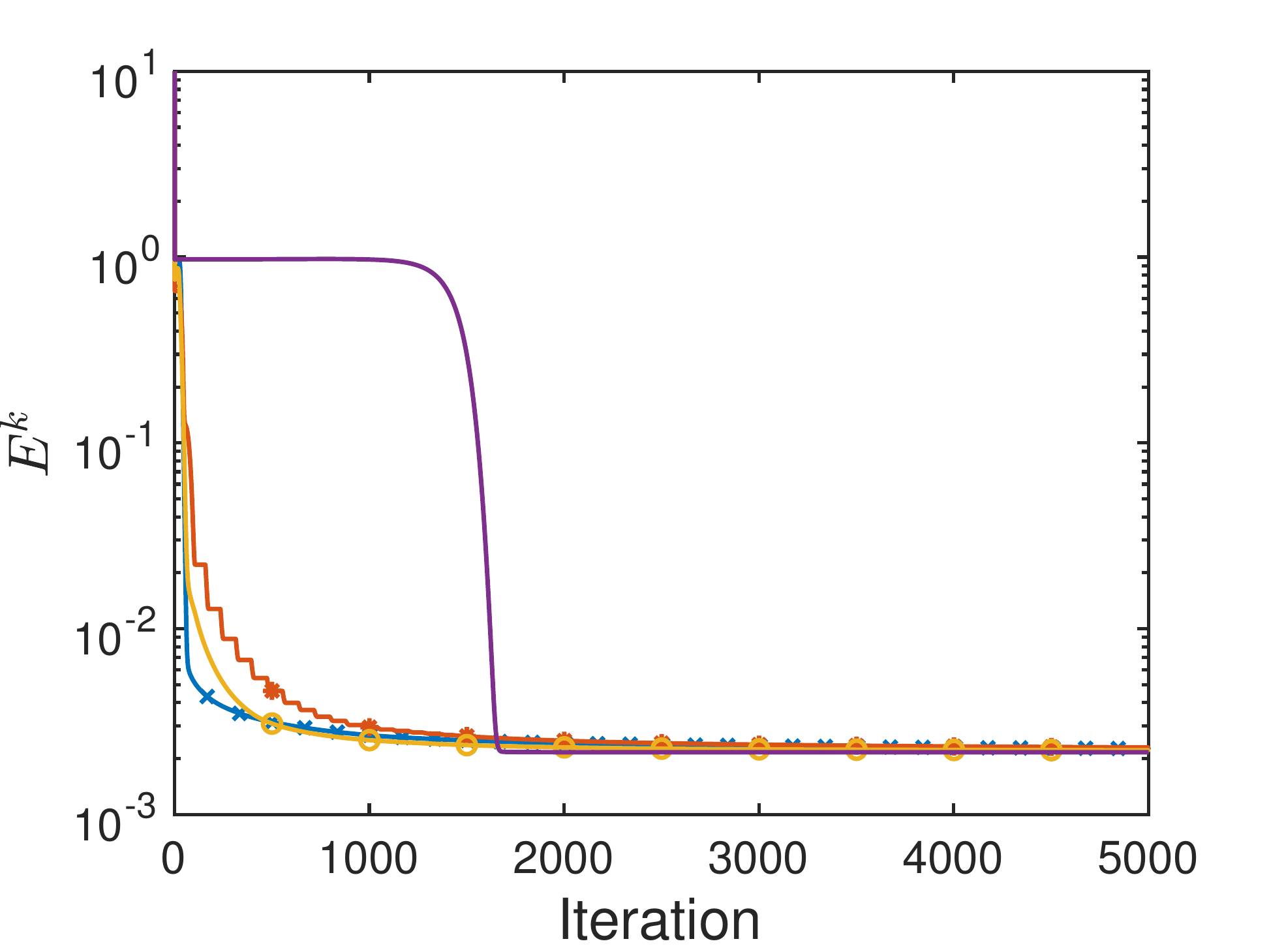}\\
		\centering{(a2) $r = 40$.}
	\end{minipage}
	\caption{\revise{Normalized fitting error $E^k$ versus iteration number on synthetic data with $n=300, ~\sigma = 0.1$, and varied factorization rank $r$.}}  \label{fig:new_vs_old}
\end{figure}

We now verify the convergence behaviors of our proposed algorithms when utilized to solve problem \eqref{eq:SNMF by reg} and   display the result in \Cref{fig:convergence f}. One can observe that  our algorithms  converge  in both noise-free ($\sigma = 0$) and noisy ($\sigma = 0.1$) cases, which corroborates our theoretical results.  It is also worth noting that A-SymHALS converges slightly faster than  SymHALS, which supports the acceleration technique used in A-SymHALS.    Interestingly, in the noise-free case, all the three proposed algorithms have a linear rate of convergence and can  find a nearly globally optimal minimizer of problem \eqref{eq:SNMF by reg} (asserted by achieving nearly zero function value)  even if the problem is  nonconvex.

\revise{ As mentioned  in \Cref{sec:related work}, the idea of solving symmetric NMF by the penalized nonsymmetric NMF \eqref{eq:SNMF by reg}  has also appeared heuristically in \cite{kuang2015symnmf}.  The authors proposed to solve  problem \eqref{eq:SNMF by reg} using an ANLS-type method, where the penalty parameter $\lambda$ is updated as $\lambda_{k+1} = 1.01\times \lambda_{k}$ until $\|\mU_k - \mV_k \|_F / \|\mV_k\|_F < 10^{-8}$. Thus, their algorithm is the same as our SymANLS except for the updating  of the penalty parameter $\lambda$.  In \Cref{fig:new_vs_old}, we compare our proposed algorithms with theirs in terms of the normalized fitting error $E^k$. Though they all eventually converge to the same $E_k$, our algorithms converge much faster than theirs. The relatively slow convergence of their algorithm is due to the fact that their updating of the penalty parameter $\lambda$ is not adaptive to the iterations of the algorithm. Using \Cref{fig:new_vs_old} (a1) as an example,  $\lambda_k$ in our algorithms is increased to around $97$ after only $85$ iterations and then remains this value for the following iterations (i.e., the adaptive updating formula \eqref{eq:update lambda} nearly converges after $85$ iterations), while their algorithm needs $1619$ iterations  to update $\lambda_{k}$ to around $97$ and keeps increasing $\lambda_k$ to $7\times 10^7$ after $3000$ iterations. 
Increasing the penalty parameter to infinity is a common strategy utilized in the analysis of methods of Lagrangian multipliers for general problems~\cite[Theorem 17.1]{nocedal2006numerical}. We avoid such a requirement by exploiting the specific structures within  problem \eqref{eq:SNMF by reg}. Thus, it is instructive to emphasize that the difference between our work and \cite{kuang2015symnmf} does not only lie in the update of the penalty parameter $\lambda$, but also on the theoretical side; see \Cref{sec:contributions} and \ref{sec:related work} for more details. In the sequel, we will not display the performance of the algorithm in \cite{kuang2015symnmf} since it has very similar performance to our SymANLS except for relatively slower convergence.}

In \Cref{fig:synthetic},  we compare  our algorithms with several state-of-the-art symmetric NMF algorithms by displaying the normalized fitting error $E^k$  versus iteration number and wall clock running time, which demonstrates the ability for solving the original symmetric NMF  \eqref{eq:SNMF}. We fix $n = 300$ and vary the factorization rank $r$ and the noise level $\sigma$.  From the top row of \Cref{fig:synthetic} (i.e., \Cref{fig:synthetic} (a1)-(a4)), it can be observed  that our proposed  SymANLS, SymHALS, and A-SymHALS outperform the others in terms of iteration number.   A-SymHALS performs the best in  the noise free settings, while SymANLS, SymHALS, and A-SymHALS have comparable performances in the noisy cases.   The bottom row of  \Cref{fig:synthetic} (i.e., \Cref{fig:synthetic} (b1)-(b4)) demonstrates the evolution of $E^k$ versus wall clock running time.  It can be observed that SymHALS and A-SymHALS have the best performances among other algorithms.  In  \Cref{fig:synthetic} (b1), where the factorization rank is small ($r = 20$) and the data is not contaminated by noise, the tSVD algorithm also performs well. However, when noise is presented and the factorization rank $r$ becomes larger, SymHALS and A-SymHALS have the best running time performance.  On the other hand, one can observe  from the experiments  where noise is added (i.e., (a2), (b2), (a4), (b4)), the  alternating-type algorithms, ADMM, and PGD can converge to solutions with almost the same fitting error $E^k$ after enough iterations, while tSVD and beta-SNMF likely get  stuck at local minima with lager fitting errors.     

\begin{figure*}[!htb]
	\centering
	\begin{minipage}{0.24\linewidth}
		\includegraphics[width=1\textwidth]{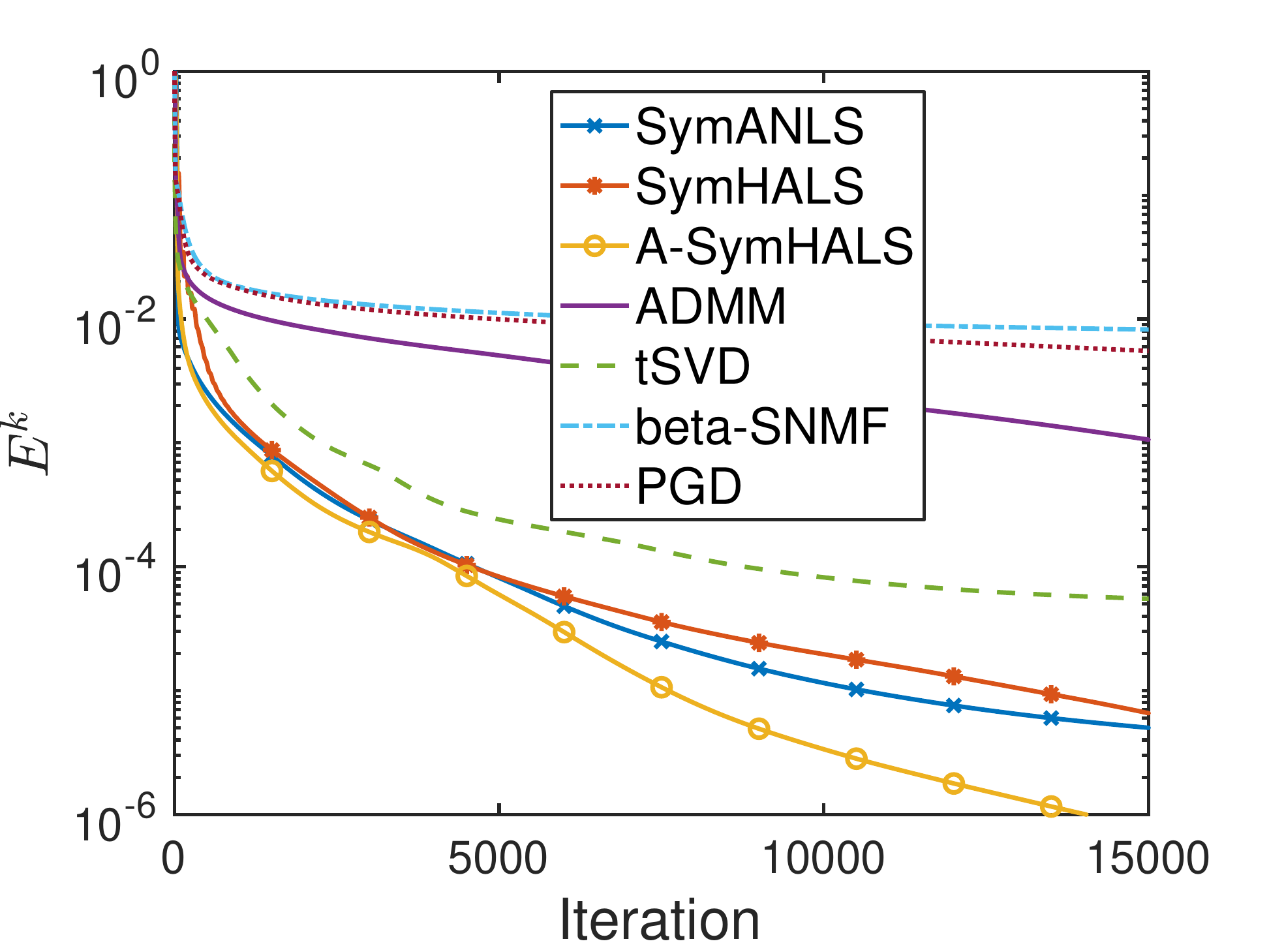}
		\centerline{(a1) $r = 20,~ \sigma = 0$}
	\end{minipage}
	\hfill
	\begin{minipage}{0.24\linewidth}
		\includegraphics[width=1\textwidth]{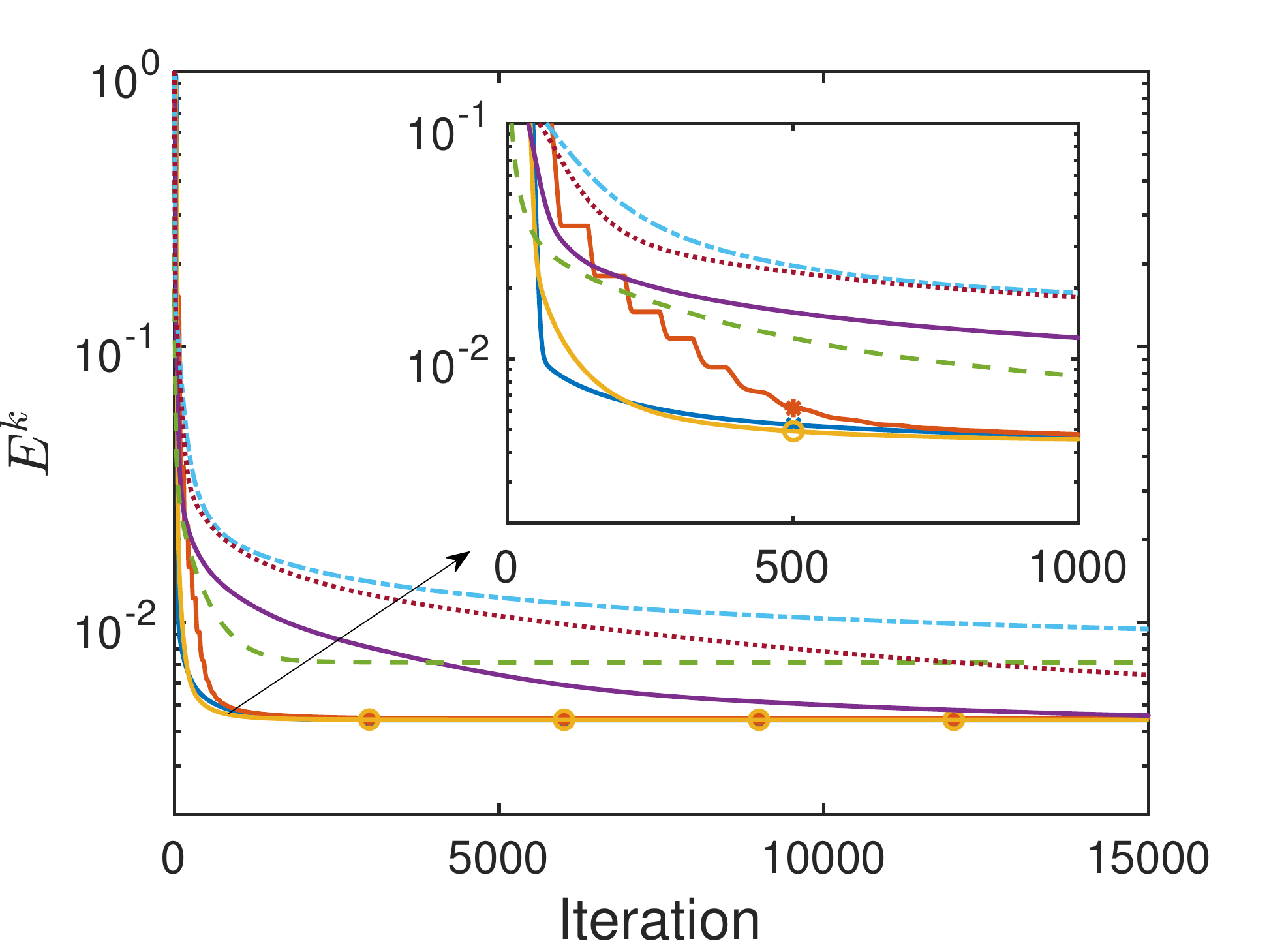}
		\centerline{(a2) $r = 20,~ \sigma = 0.1$}
	\end{minipage}
	\begin{minipage}{0.24\linewidth}
		\includegraphics[width=1\textwidth]{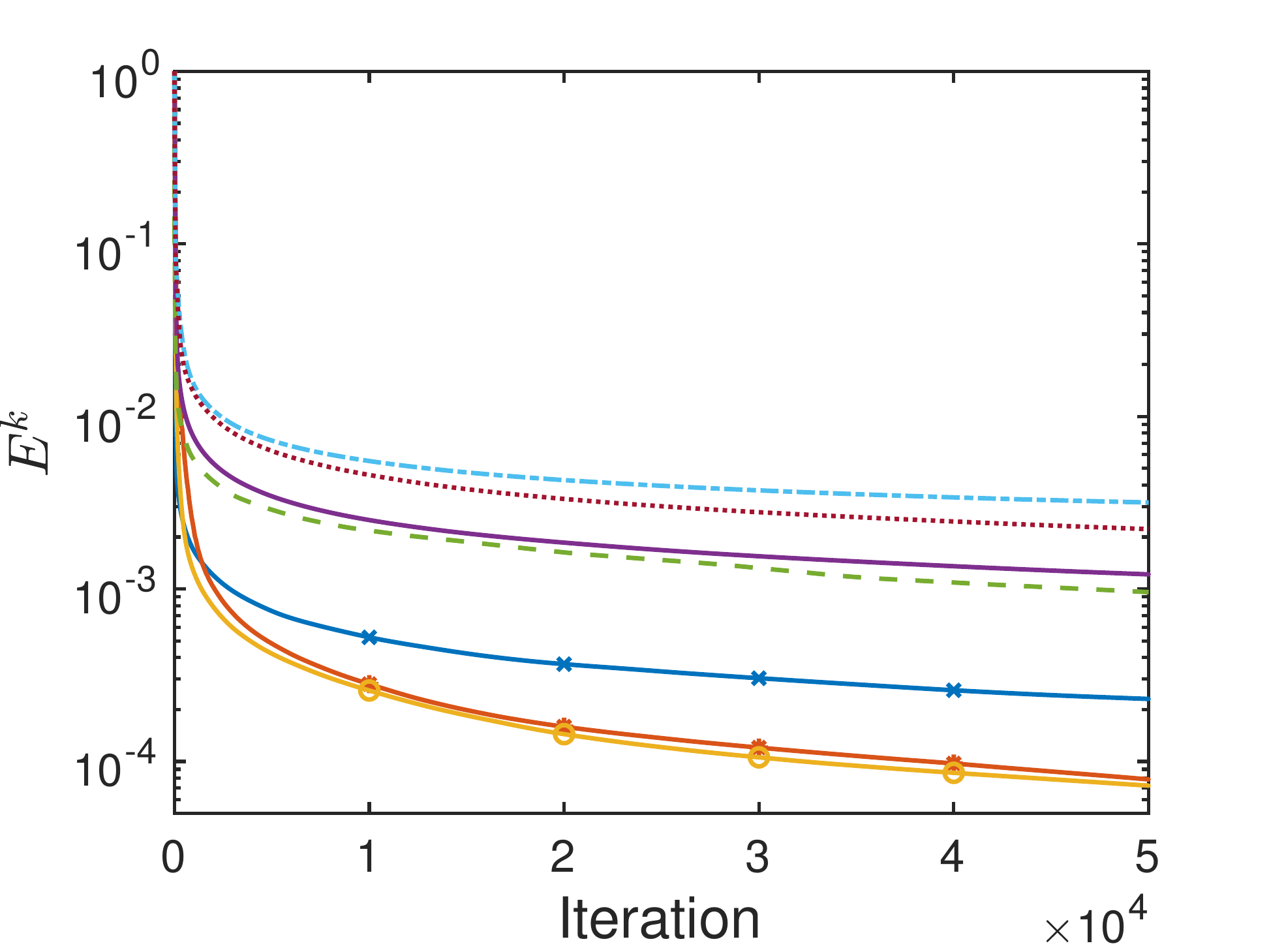}
		\centerline{(a3) $r = 40,~ \sigma = 0$}
	\end{minipage}
	\hfill
	\begin{minipage}{0.24\linewidth}
		\includegraphics[width=1\textwidth]{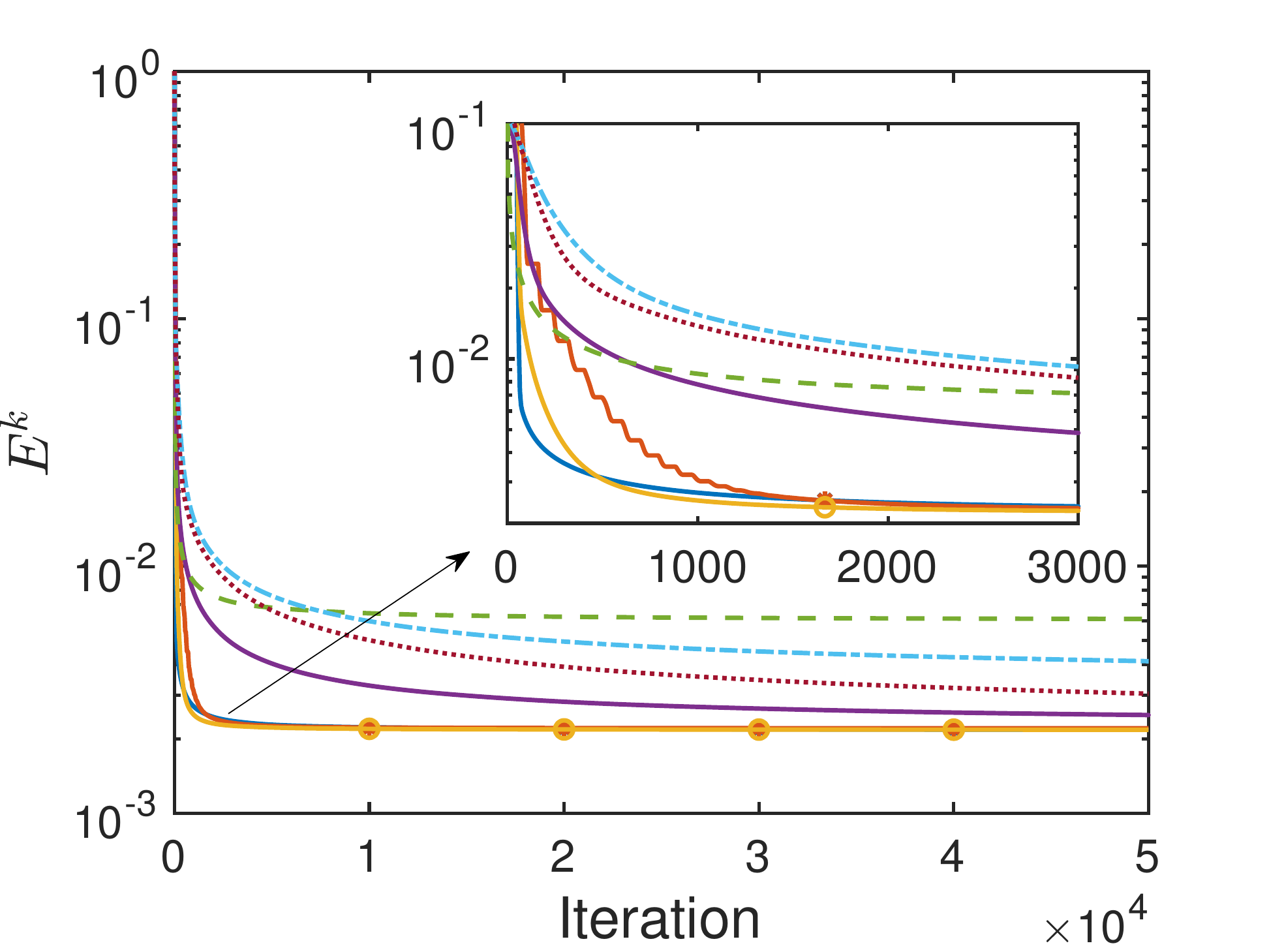}
		\centerline{(a4) $r = 40,~ \sigma = 0.1$}
	\end{minipage}
	\vfill	
	\begin{minipage}{0.24\linewidth}
		\includegraphics[width=1\textwidth]{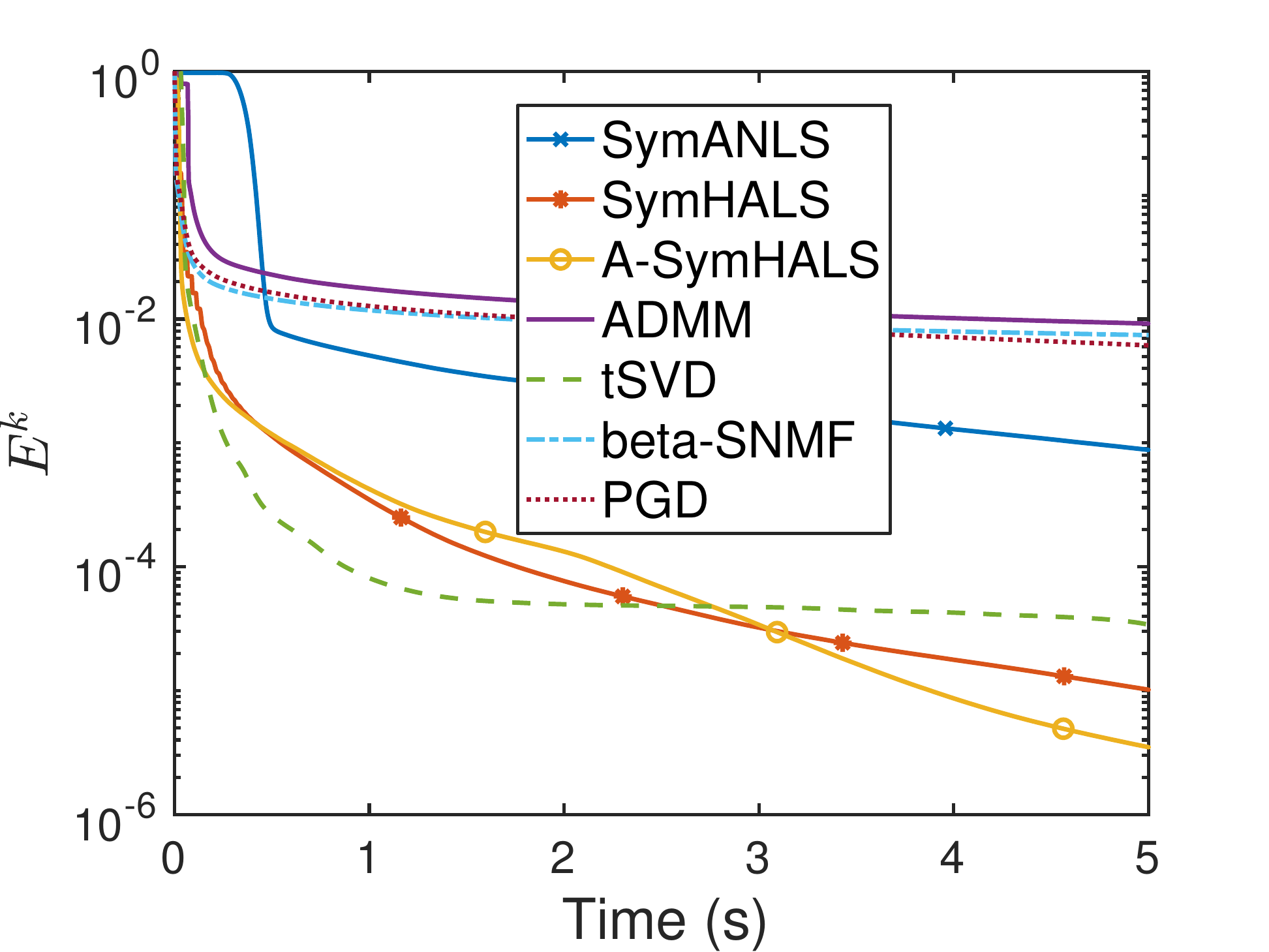}
		\centerline{(b1) $r = 20,~ \sigma = 0$}
	\end{minipage}
	\hfill
	\begin{minipage}{0.24\linewidth}
		\includegraphics[width=1\textwidth]{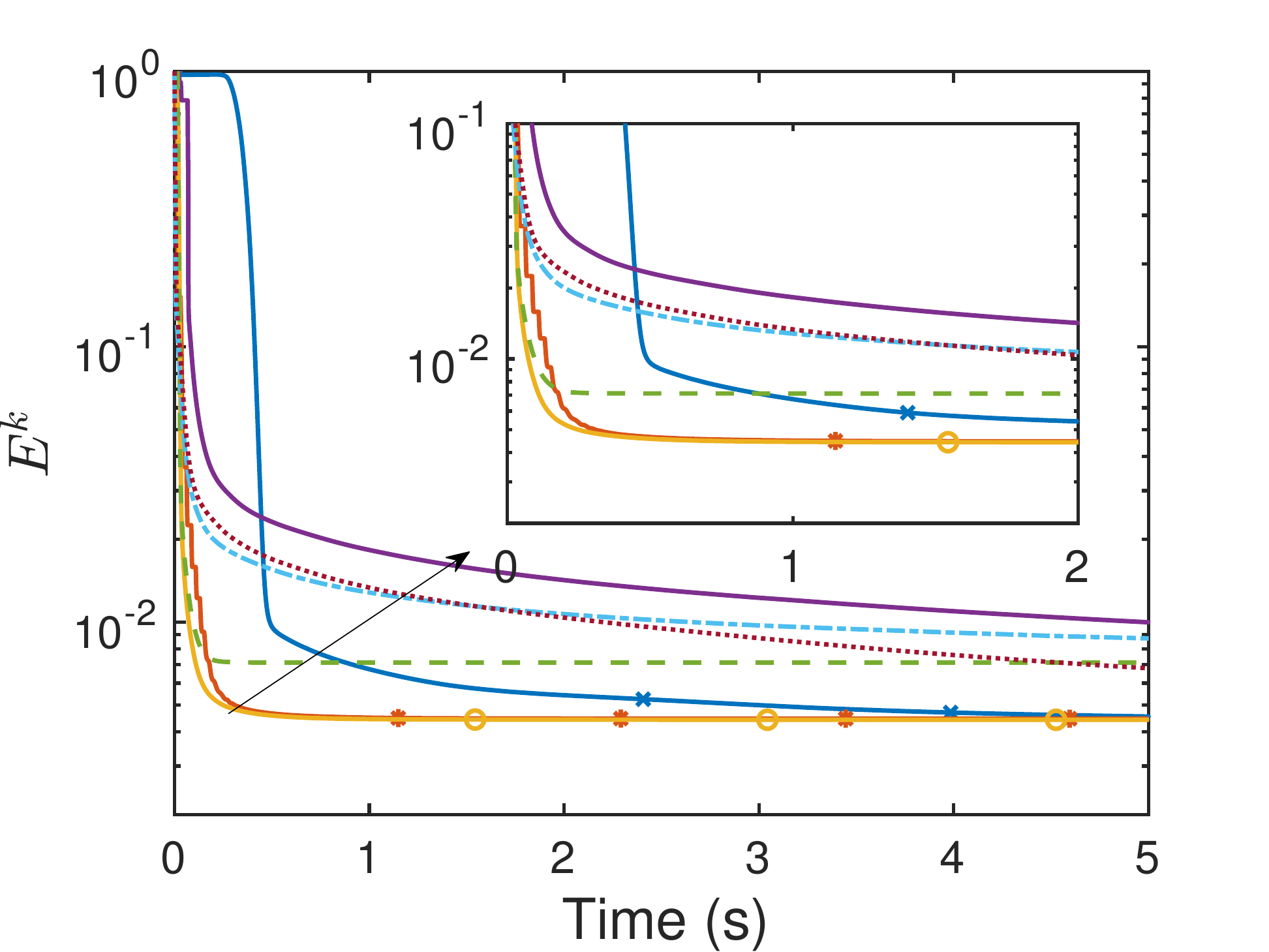}
		\centerline{(b2) $r = 20,~ \sigma = 0.1$}
	\end{minipage}
	\begin{minipage}{0.24\linewidth}
		\includegraphics[width=1\textwidth]{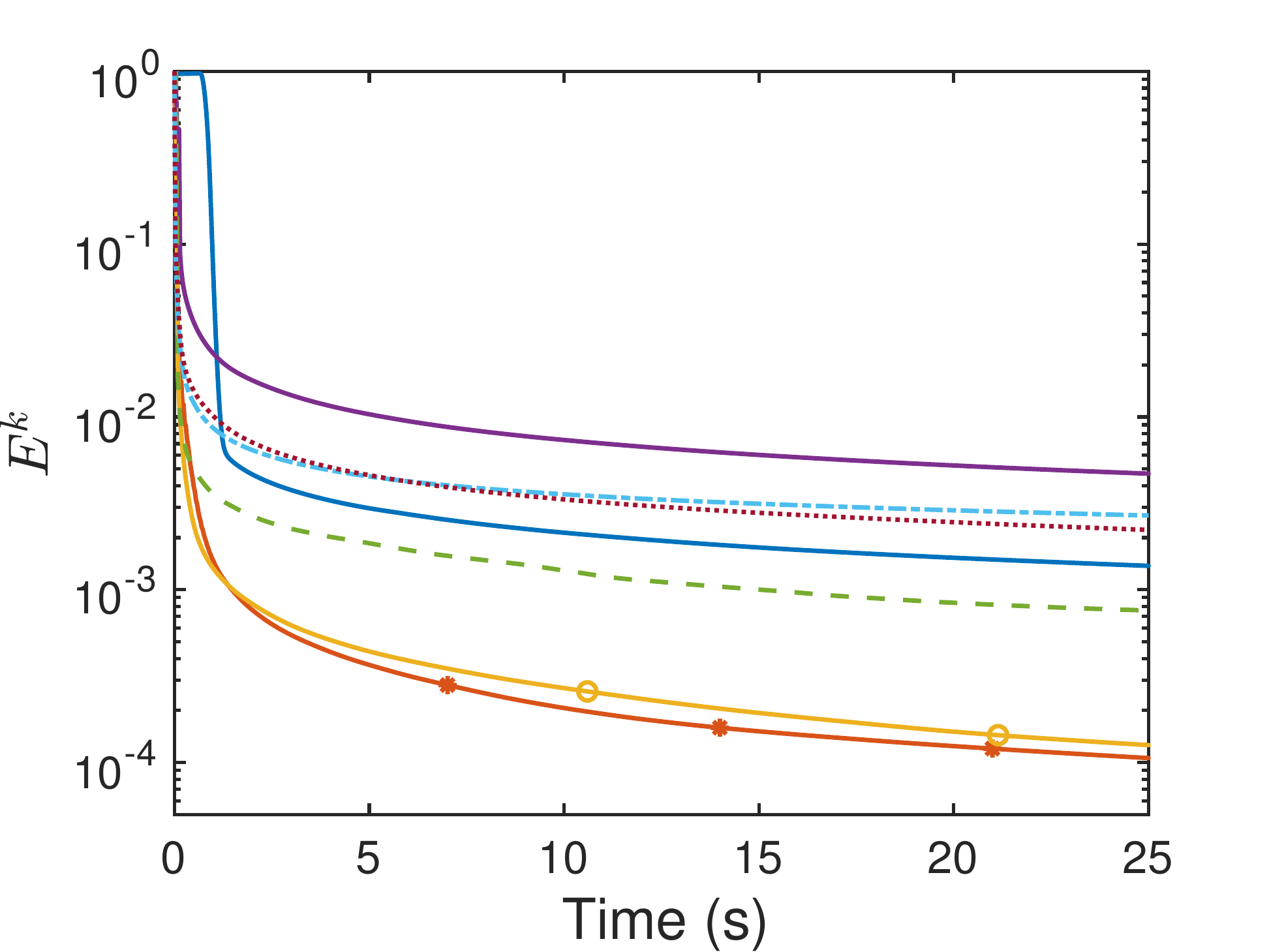}
		\centerline{(b3) $r = 40,~ \sigma = 0$}
	\end{minipage}
	\hfill
	\begin{minipage}{0.24\linewidth}
		\includegraphics[width=1\textwidth]{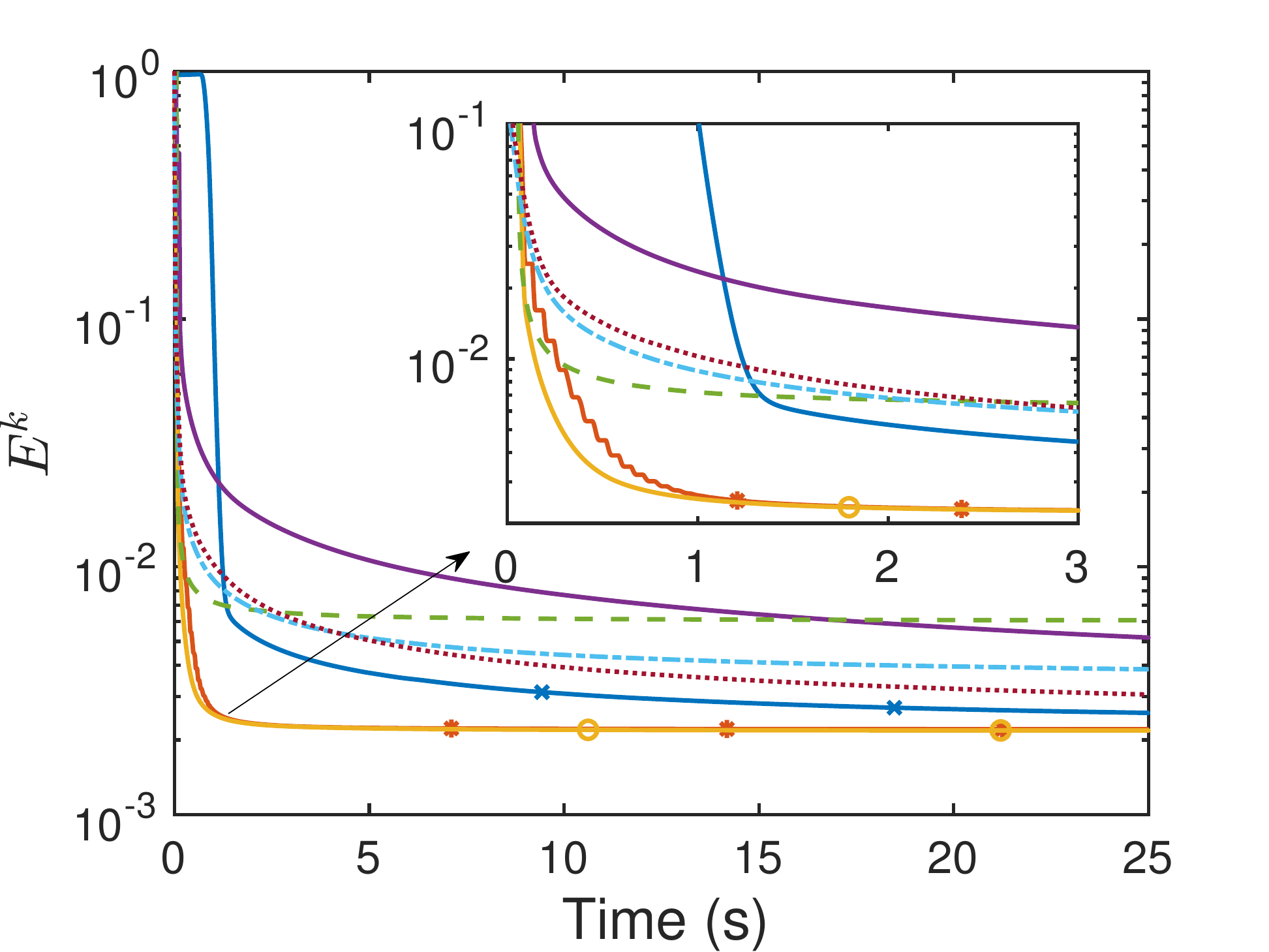}
		\centerline{(b4) $r = 40,~ \sigma = 0.1$}
	\end{minipage}
	\caption{ Normalized fitting error $E^k$ versus iteration number (top row) and wall clock running time (bottom row) on synthetic data with $n= 300$, varied factorization rank $r$ and noise level $\sigma$. } \label{fig:synthetic}
\end{figure*}

\begin{figure*}[!htb]
	\centering
	\begin{minipage}{0.24\linewidth}
		\includegraphics[width=1\textwidth]{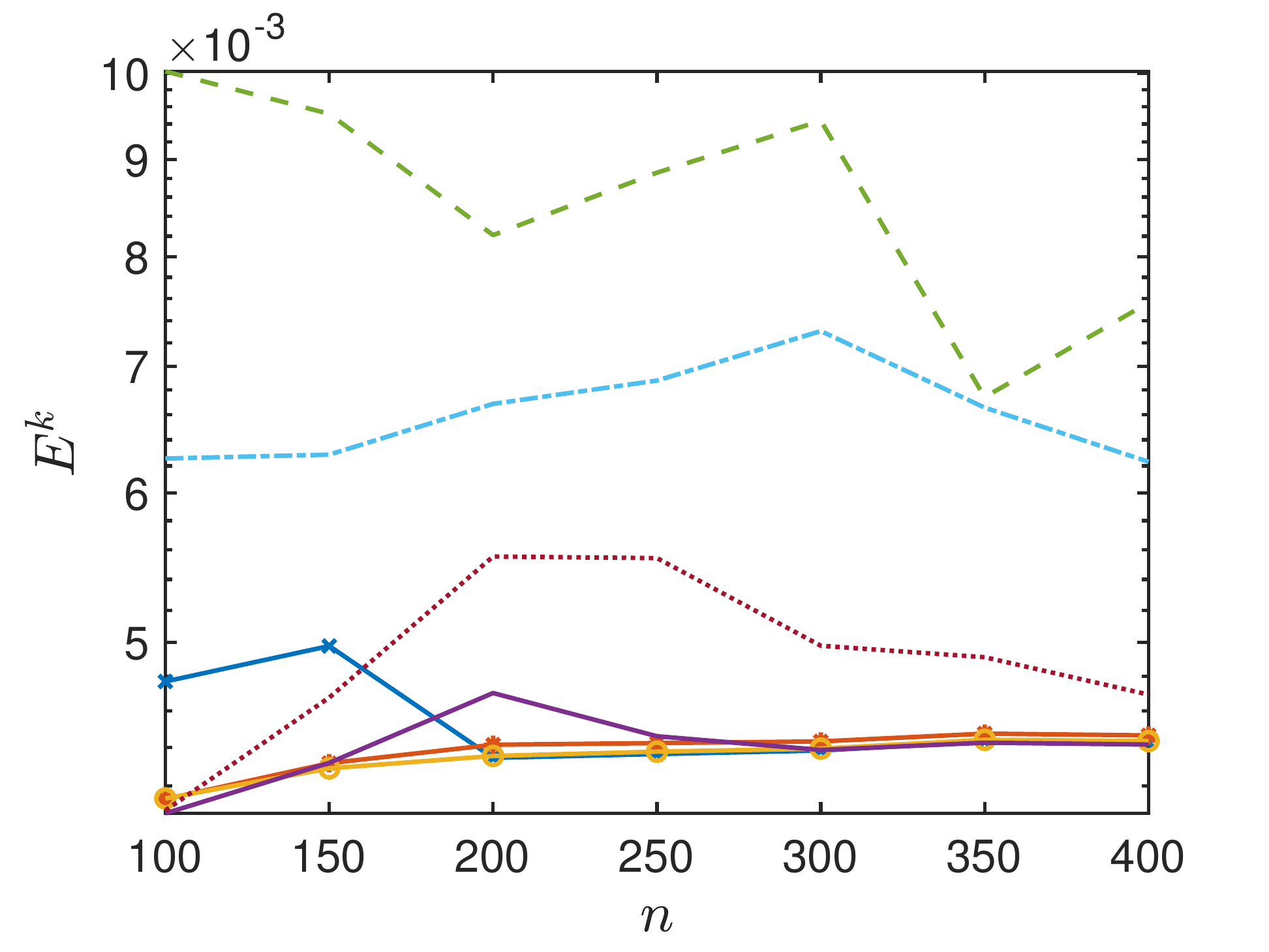}
		\centerline{(a1) $r = 20, ~ \sigma = 0.1$}
	\end{minipage}
	\hfill
	\begin{minipage}{0.24\linewidth}
		\includegraphics[width=1\textwidth]{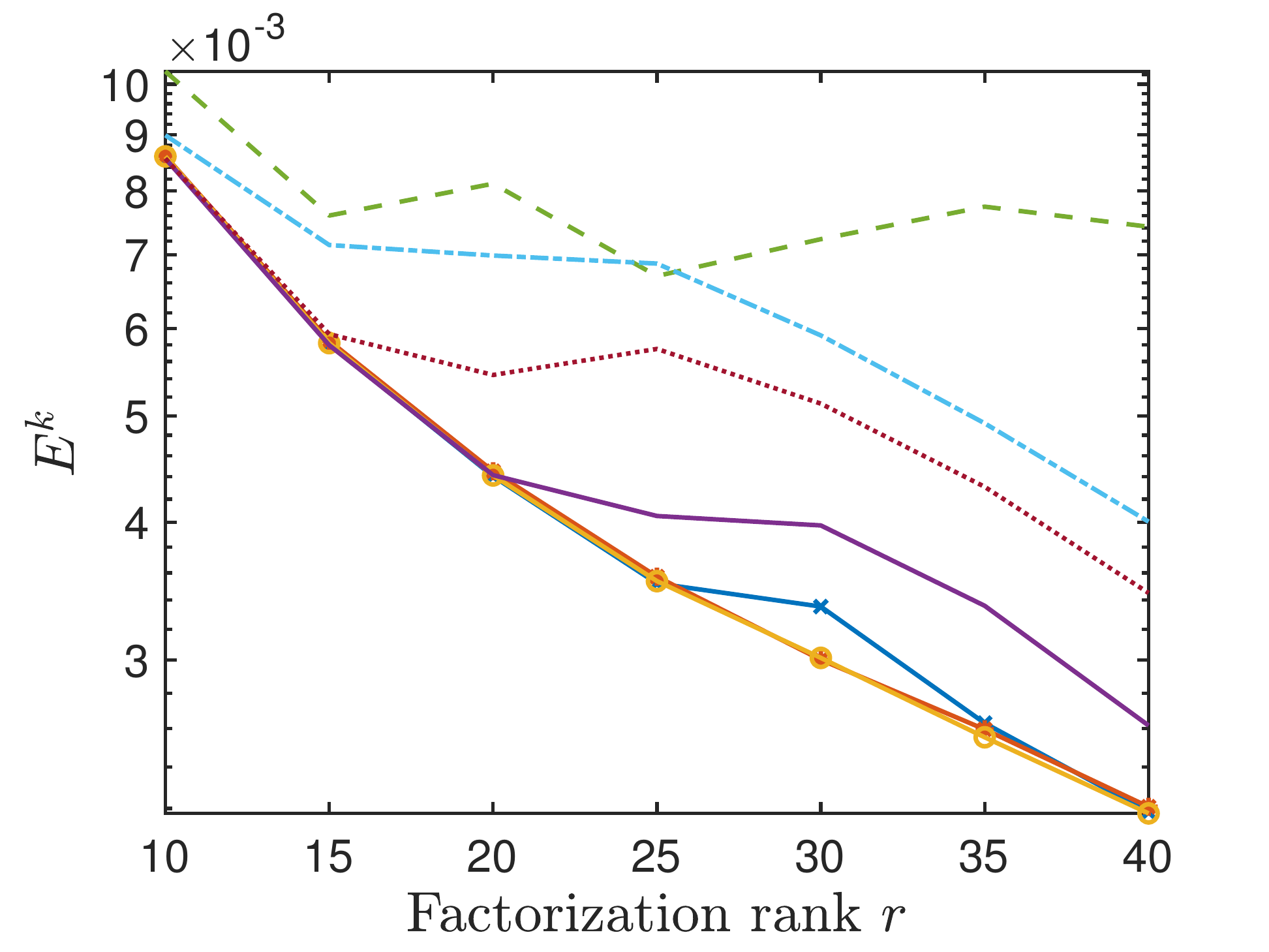}
		\centerline{(a2) $n = 300, ~\sigma = 0.1$}
	\end{minipage}
	\begin{minipage}{0.24\linewidth}
		\includegraphics[width=1\textwidth]{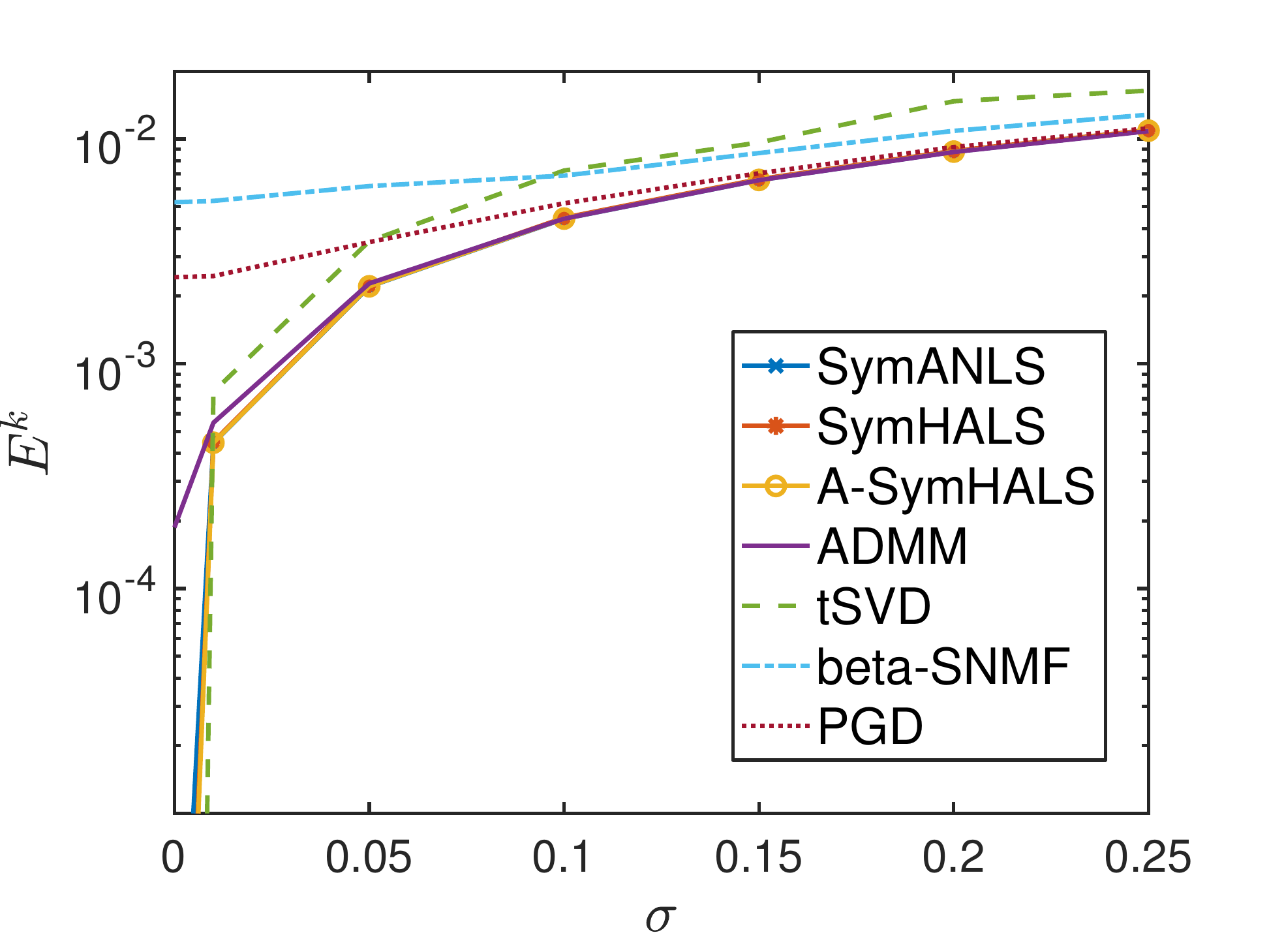}
		\centerline{(a3) $n= 300, ~r = 20$}
	\end{minipage}
	\hfill
	\begin{minipage}{0.24\linewidth}
		\includegraphics[width=1\textwidth]{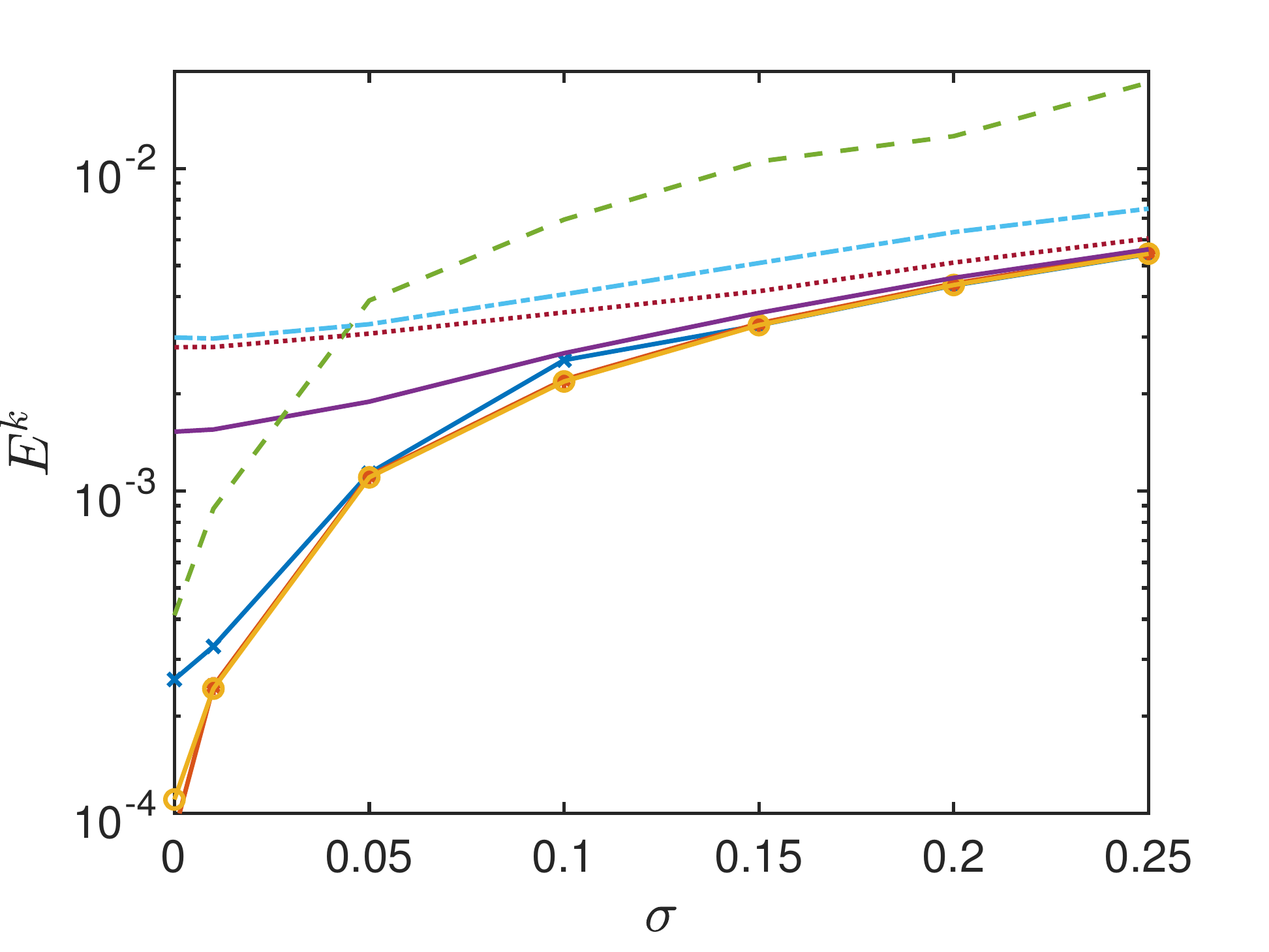}
		\centerline{(a4) $n=300, ~r = 40$}
	\end{minipage}
	\caption{  \revise{Normalized fitting error $E^k$  on synthetic data by varying experimental parameters. (a1):  data size  $n\in \{100, 150, 200, 250, 300, 350, 400\}$. (a2): Factorization rank $r\in \{ 10, 15, 20, 25, 30, 35, 40 \}$. (a3): Noise level $\sigma\in \{0, 0.01, 0.05, 0.1, 0.15, 0.2, 0.25\}$. (a4): Noise level $\sigma\in \{0, 0.01, 0.05, 0.1, 0.15, 0.2, 0.25\}$.}} \label{fig:vary}
\end{figure*}

\begin{figure*}[!htb]
	\centering
	\begin{minipage}{0.24\linewidth}
		\includegraphics[width=1\textwidth]{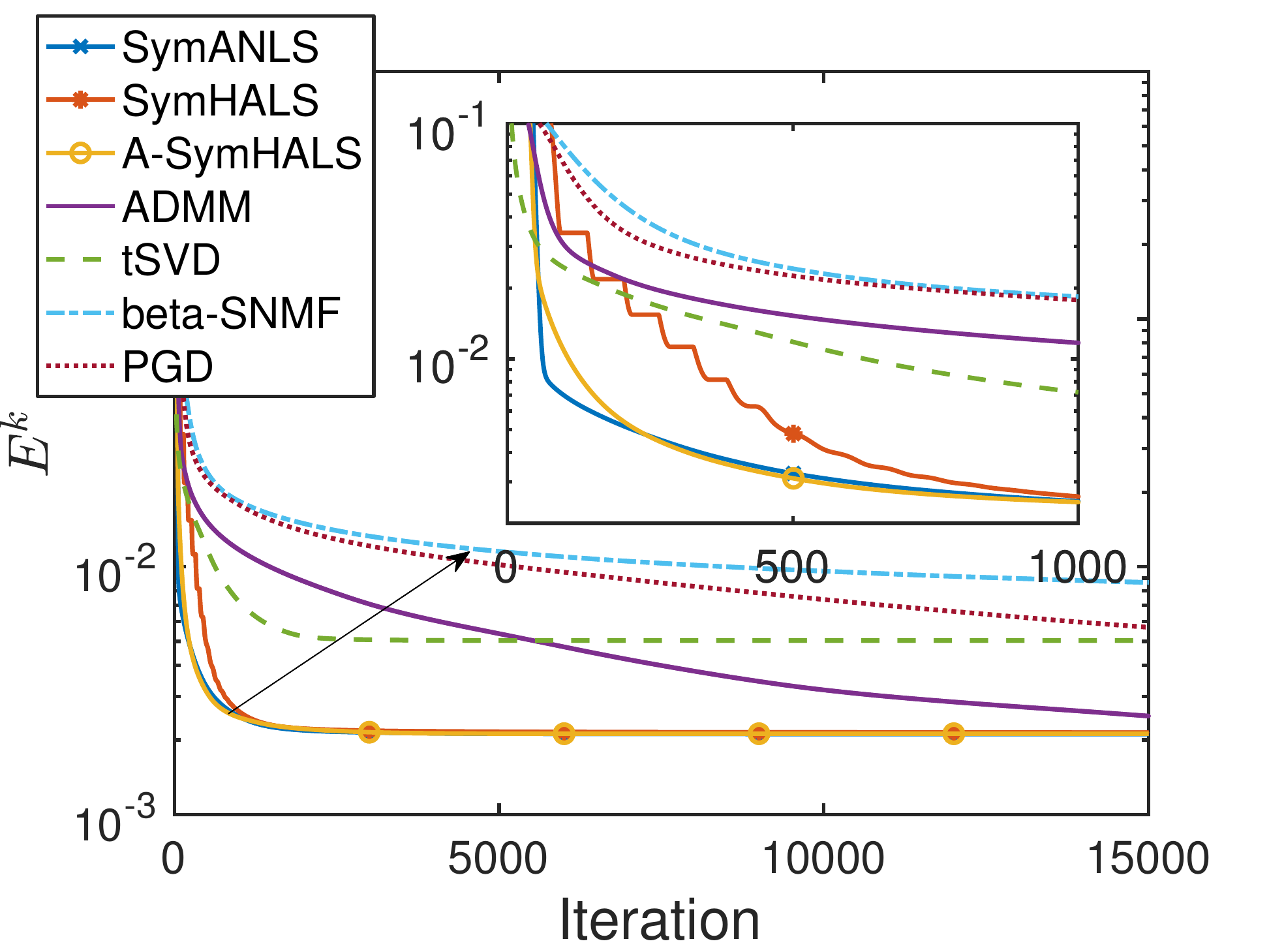}
		\centerline{(a1) $r = 20$, uniform noise}
	\end{minipage}
	\hfill
	\begin{minipage}{0.24\linewidth}
		\includegraphics[width=1\textwidth]{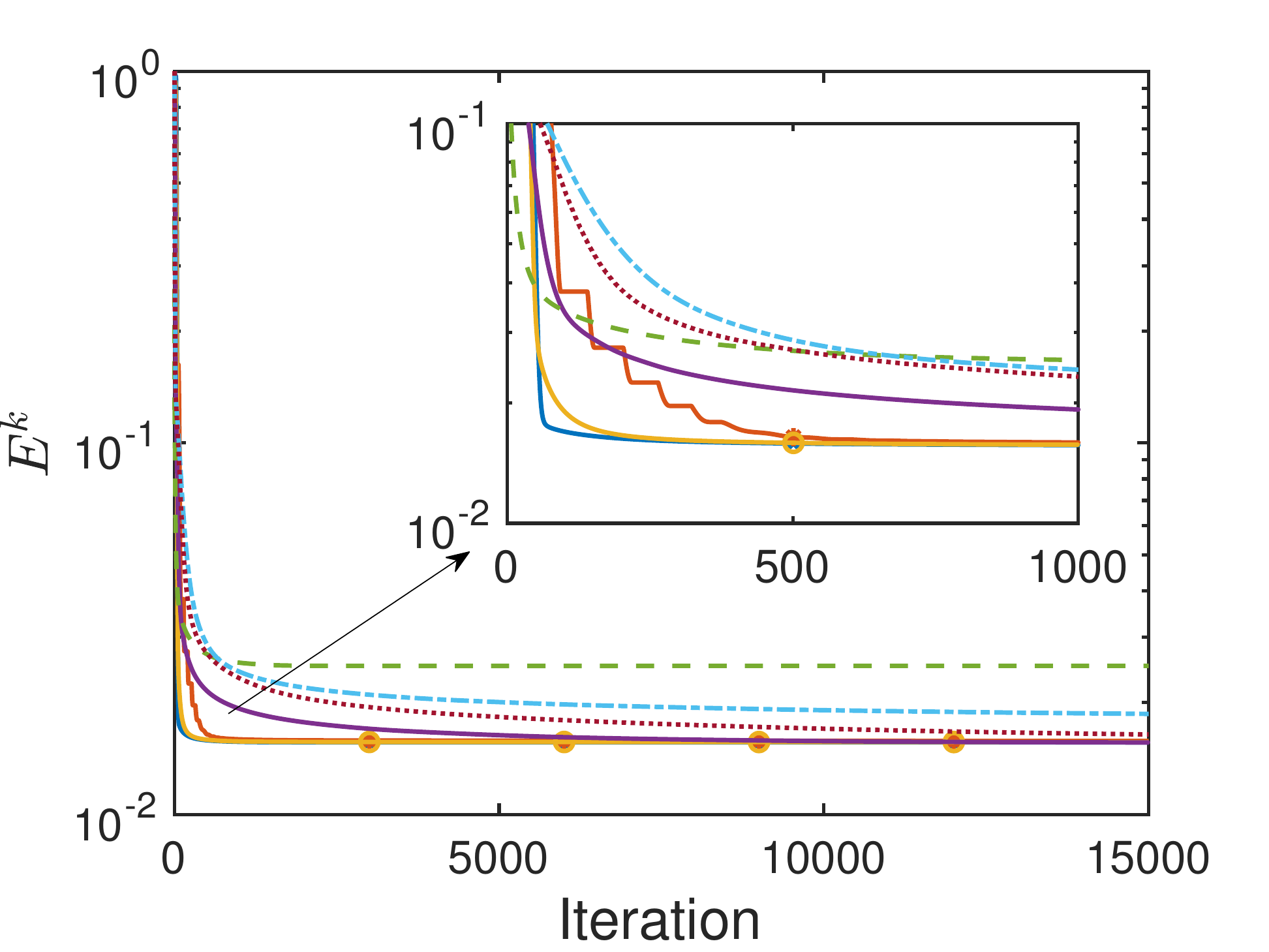}
		\centerline{(a2) $r = 20$, lognormal noise}
	\end{minipage}
	\begin{minipage}{0.24\linewidth}
		\includegraphics[width=1\textwidth]{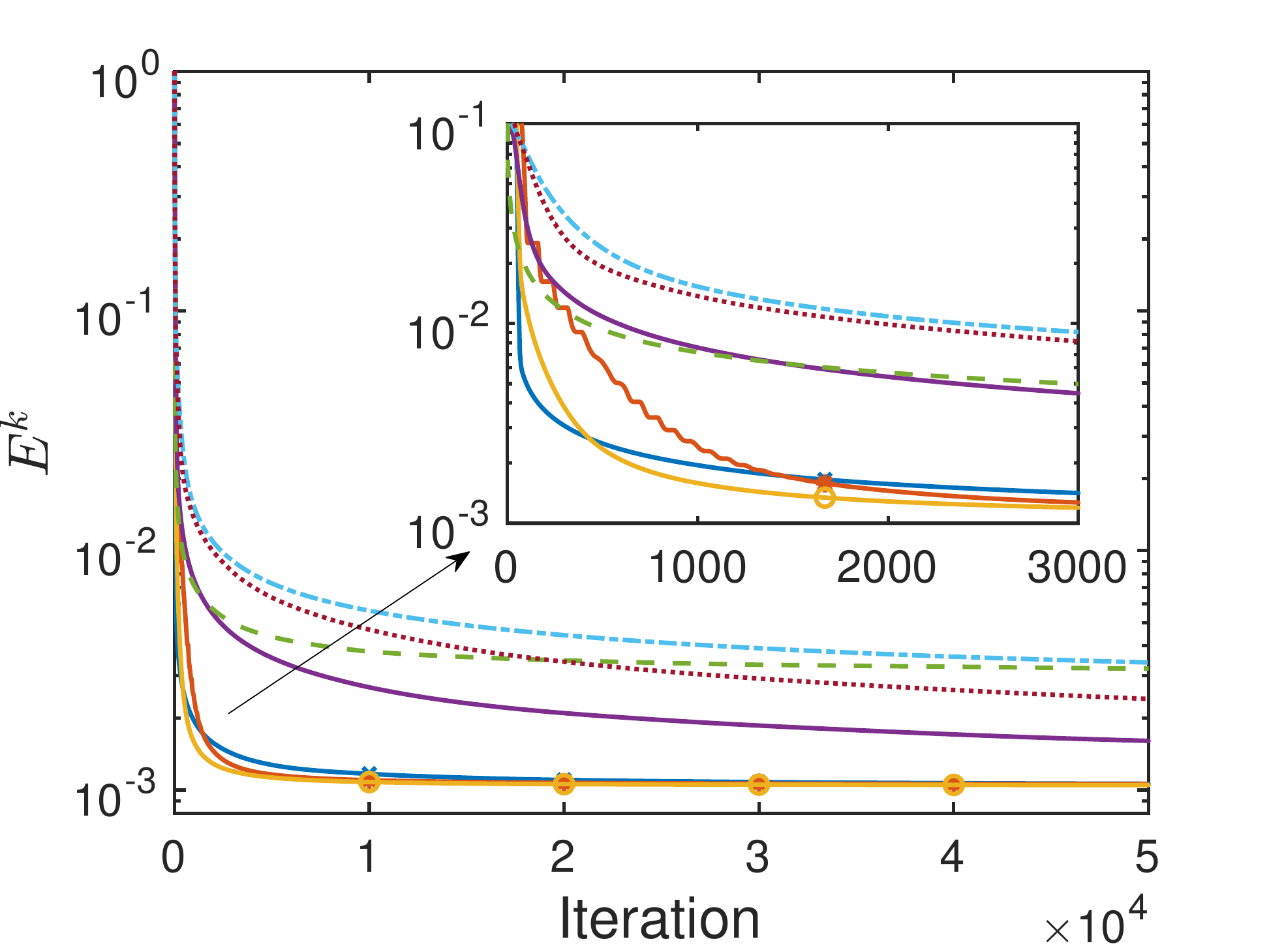}
		\centerline{(a3) $r = 40$, uniform noise}
	\end{minipage}
	\hfill
	\begin{minipage}{0.24\linewidth}
		\includegraphics[width=1\textwidth]{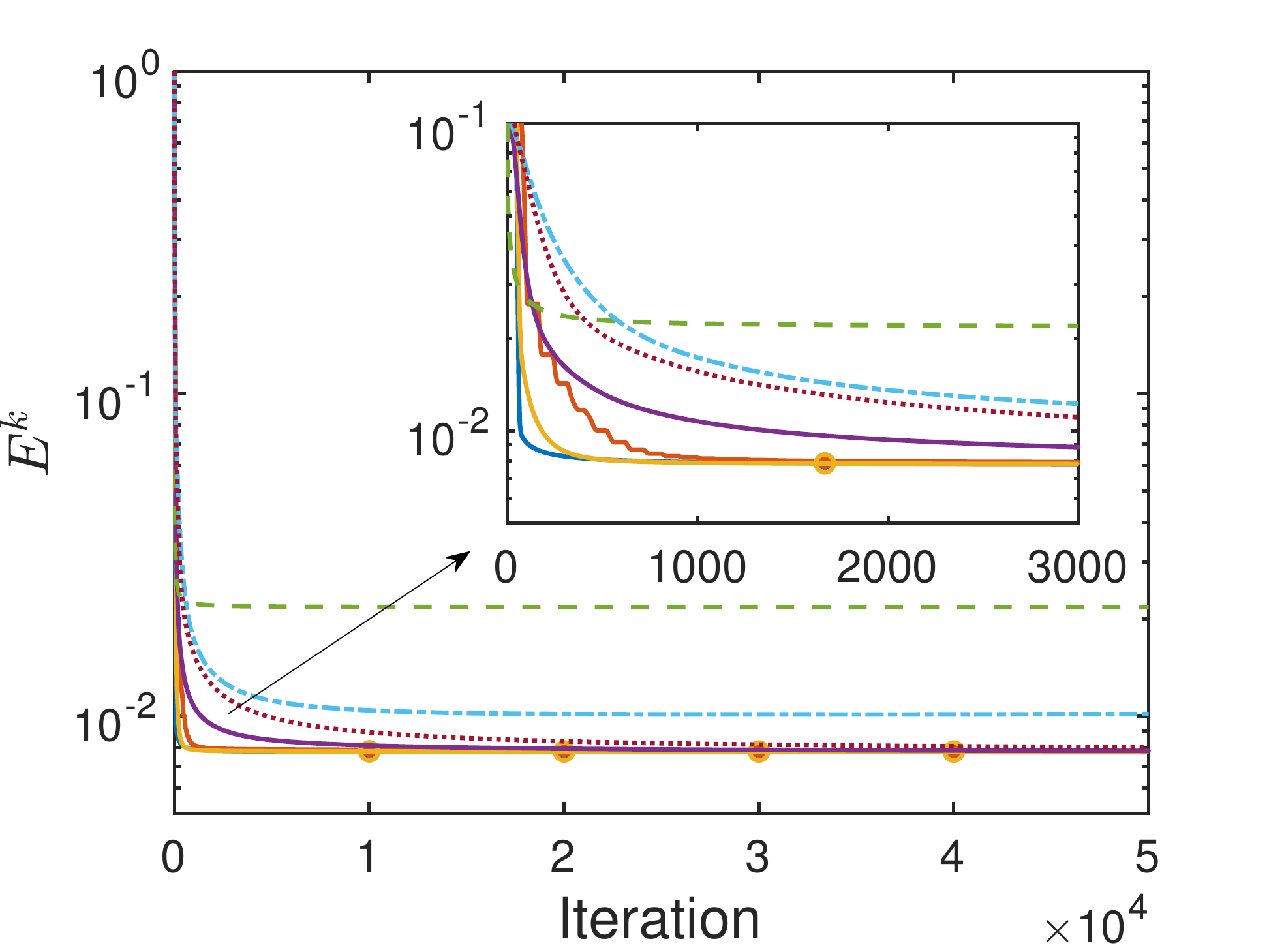}
		\centerline{(a4) $r = 40$, lognormal noise}
	\end{minipage}
	\caption{  \revise{Normalized fitting error $E^k$ versus iteration number  on synthetic data with $n= 300$, noise level $\sigma=0.1$, varied factorization rank $r$, and varied noise distributions. }} \label{fig:noise}
\end{figure*}

\revise{To further investigate the performance of our proposed algorithms, we vary several experimental parameters and show the results in \Cref{fig:vary}.  We run each algorithm  $3\times 10^4$  iterations for each parameter setting in order to ensure convergence.  In \Cref{fig:vary} (a1), we fix $r = 20, ~\sigma = 0.1$ and vary $n\in \{100, 150, 200, 250, 300, 350, 400\}$. It is observed that varying $n$  while keeping the other parameters fixed does not affect too much the performance of each algorithm. This is reasonable as varying $n$ has almost no effect to the signal-to-noise ratio. We can conclude from \Cref{fig:vary} (a1) that our algorithms and ADMM outperform others for different $n$. In \Cref{fig:vary} (a2), we fix $n = 300, ~\sigma = 0.1$ and vary $r\in \{ 10, 15, 20, 25, 30, 35, 40 \}$. It is clear that our algorithms outperform others for almost all choices of $r$. We also vary the noise level $\sigma\in \{0, 0.01, 0.05, 0.1, 0.15, 0.2, 0.25\}$ in \Cref{fig:vary} (a3) (where $n=300, ~ r= 20$) and \Cref{fig:vary} (a4)  (where $n=300, ~ r= 40$). We observe that our algorithms perform better than others for almost all choices of $\sigma$. }

\revise{From \Cref{fig:synthetic} and \ref{fig:vary}, we can observe the robustness of our proposed algorithms to Gaussian noise. We further exam the performance on other types of noise. With the same experimental settings as used to generate \Cref{fig:synthetic} (except for the noise distribution), we test all algorithms on synthetic data contaminated by noise generated according to the uniform distribution and the lognormal distribution; see \Cref{fig:noise}.  We can observe similar phenomena  that our proposed algorithms  perform better than others in these cases.}

\begin{figure*}[!htb]
	\centering
	\begin{minipage}{0.49\linewidth}
		\includegraphics[width=1\textwidth]{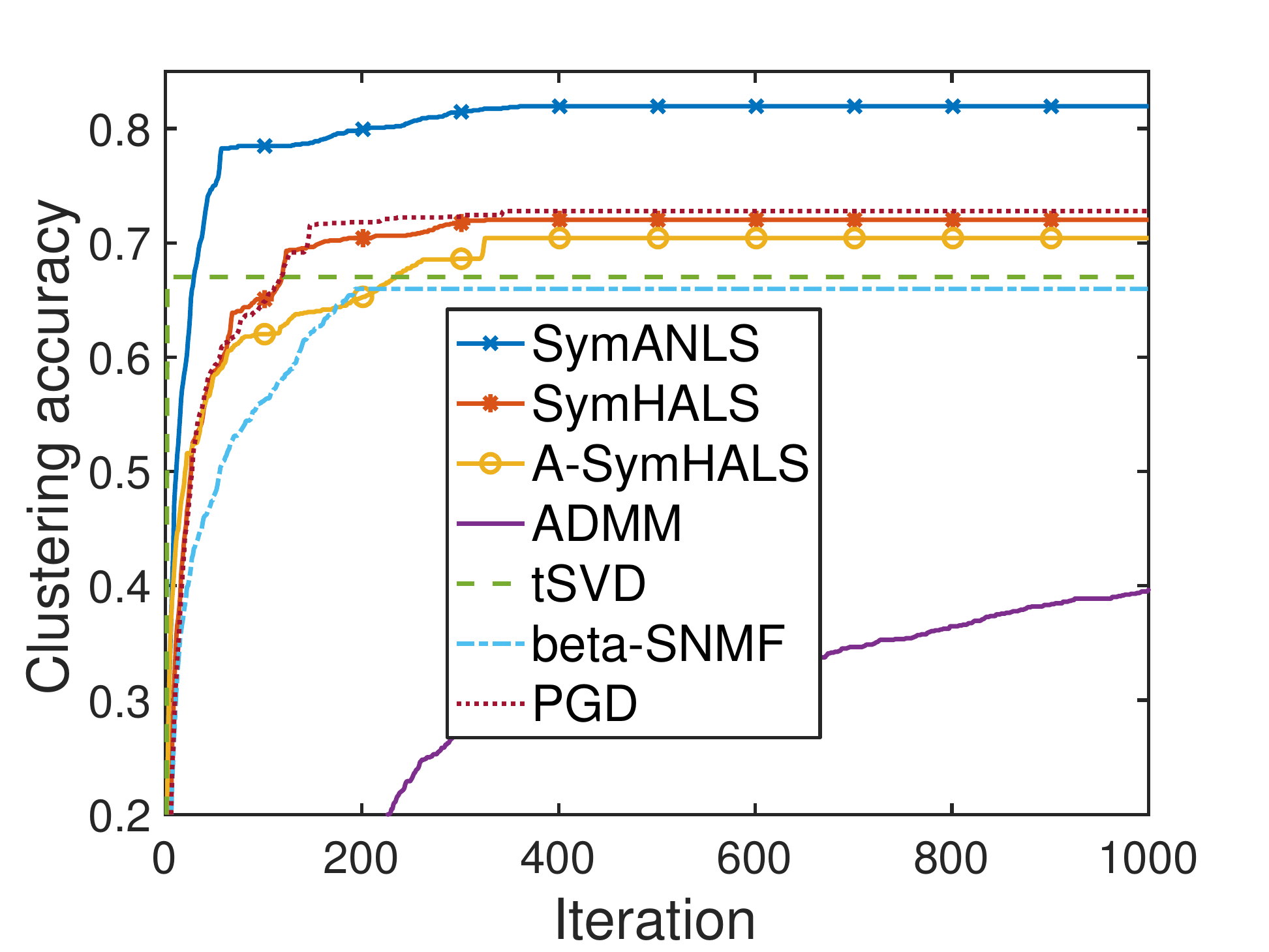}\\
		\centering{(a1) COIL-20 dataset, $n = 1440,~r = 20$}
	\end{minipage}
	\hfill
	\begin{minipage}{0.49\linewidth}
		\includegraphics[width=1\textwidth]{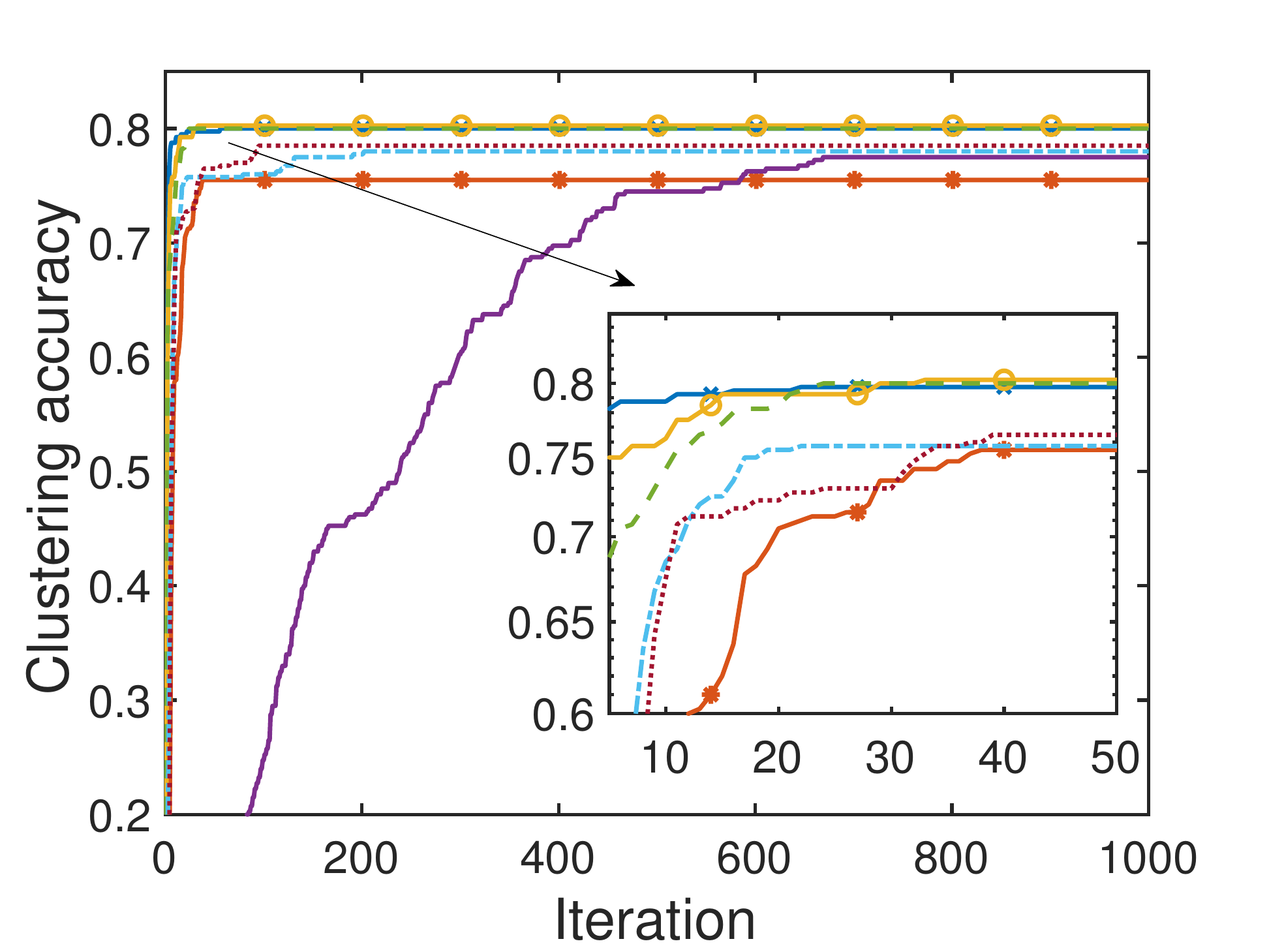}\\
		\centering{(a2) ORL dataset, $n = 400, ~r = 40$.}
	\end{minipage}
	\caption{Image clustering accuracy versus iteration number on real dataset. }  \label{fig:real date}
\end{figure*}
%

\subsection{Image Clustering}

Symmetric NMF can be used for graph clustering where each element $\mX_{ij}$ denotes the similarity between data $i$ and $j$ \cite{kuang2015symnmf,ding2005equivalence}.
In this subsection, we apply different symmetric NMF algorithms for graph clustering on real image datasets and compare the clustering accuracy~\cite{xu2003document}.\footnote{\revise{Note that there exist many other clustering methods, but comparing with them is not the focus of this paper. Instead, we only compare different symmetric NMF algorithms to demonstrate the performance of the proposed algorithms for solving symmetric NMF.}}

We first put all images to be clustered in a data matrix $\mM$, where each row is a vectorized image.  We then construct the similarity matrix following the procedures in ~\cite[section 7.1, step 1 to step 3]{kuang2015symnmf}, and utilize self-tuning method to form the similarity matrix $\mX$.
Upon deriving $\widetilde \mU$ from symmetric NMF  $\mX \approx \widetilde \mU\widetilde \mU^\T$, the label of the $i$-th image can be obtained by:
\begin{equation}\label{eq:assign_label}
label({\mM}_i) = \argmax_{j}{\widetilde \mU}_{(ij)}.
\end{equation}
We  conduct the experiments on four image datasets:

\begin{table}[!t]\caption{Summary of image clustering accuracy of different algorithms on five image datasets}\label{table:image clustering}
	\vspace{-.2in}
	\setlength{\tabcolsep}{4pt}
	\renewcommand{\arraystretch}{1}
	\begin{center}
		\resizebox{0.5\textwidth}{!}{
			\begin{tabular}{c|c|c|c|c|c}
				\hline
				&ORL&COIL-20&MNIST$_{train}$&TDT2&MNIST$_{test}$\\
				\hline
				SymANLS&{0.8000}&\textbf{0.8194}&0.6217&0.9793&0.8589\\
				\hline
				SymHALS&0.7550&0.7201&{0.6393}&{0.9800}&{0.8589}\\
				\hline
				A-SymHALS&\textbf{0.8025}&{0.7042}&\textbf{0.7043}&\textbf{0.9803}&{0.8589}\\
				\hline
				ADMM&0.7750&0.6937&0.5803&0.9800&{0.8713}\\
				\hline
				tSVD&0.8000&0.6701&{0.6653}&{0.6044}&{0.9050}\\
				\hline
				beta-SNMF&0.7800&0.6597&0.5100&0.9647&0.8513\\
				\hline
				PGD&0.7850&0.7278&0.6287&0.8313&\bf{0.9136}\\
				\hline
			\end{tabular}
		}
	\end{center}
\end{table}

\textbf{ORL}: 400 facial images from 40 different persons with each one having 10 images from different angles and emotions\footnote{http://www.cl.cam.ac.uk/research/dtg/attarchive/facedatabase}.

\textbf{COIL-20}: 1440 images from 20 objects\footnote{http://www.cs.columbia.edu/CAVE/software/softlib/coil-20.php}.

\textbf{TDT2}: 10,212 news articles from 30 categories\footnote{https://www.ldc.upenn.edu/collaborations/past-projects}. We extract the first 3147 data for experiments (containing only 2 categories).

\textbf{MNIST}: classical handwritten digits dataset\footnote{http://yann.lecun.com/exdb/mnist/}, from which we take the first 3147 images from each of the training and test sets.

In \Cref{fig:real date} (a1), we display  the  clustering accuracy on dataset \textbf{COIL-20} versus iteration number.  Similar results for dataset \textbf{ORL} are plotted in  \Cref{fig:real date} (a2).  
We observe that though tSVD and beta-SNMF have fast convergence speed, they often provide worse clustering accuracies compared to the proposed methods, which is consistent with the conclusion drawn in the noisy synthetic  experiments. We note that the performance of ADMM will increase as iteration goes and after almost 6000 iterations it reaches a comparable result to other algorithms on COIL-20 dataset. Moreover, it requires more iterations for larger dataset. This observation makes ADMM less favorable for image clustering due to its computational burden.  These results as well as the experimental results shown in the last subsection demonstrate $(i)$ the power of transferring the symmetric NMF \eqref{eq:SNMF} to a penalized nonsymmetric one \eqref{eq:SNMF by reg}; and $(ii)$ the efficiency of alternating-type algorithms for solving \eqref{eq:SNMF by reg} by exploiting the splitting property within the optimization variables in \eqref{eq:SNMF by reg}.


Finally, Table~\ref{table:image clustering} shows the clustering accuracies of the algorithms on different datasets, where  we run enough iterations for ADMM so that it obtains its best results. We observe from Table~\ref{table:image clustering} that SymANLS, SymHALS and A-SymHALS perform better than or have comparable performance to the others in most of the cases.

%

\section{ Proof of \Cref{thm:convergence}} \label{Appendix}
Since both SymANLS and SymHALS can be viewed as special cases of A-SymHALS in terms of convergence, we will only focus on the proof of the convergence for A-SymHALS. According to \Cref{thm:critical point U = V}, the remaining task is to establish the descent property as well as the convergence behavior  of A-SymHALS to a critical point of the penalized nonsymmetric NMF  \eqref{eq:SNMF by reg}. As stated in \Cref{remark: proof of A-SymHALS}, the multiple update scheme in A-SymHALS destroys the possibility to directly apply the disciplined KL convergence analysis framework.  Due to the multiple update scheme, we can only obtain a weaker version of the so-called safeguard property in the standard framework (see \Cref{lem:safeguard 2}), but a slightly strengthened sufficient decrease  property (see \Cref{lem:sufficient decrease 2}).

\subsection{Notations Used in The Proof}\label{subsec:notations in appendix}
We stack $\mU$ and $\mV$ into one variable $\mW:=(\mU,\mV)$. We may constantly change between the  notations $\mW$ and $(\mU,\mV)$. Let $\mW_{ k}^j = (\mU_{k}^j,\mV_{k}^j)$  represent the $j$-th inner iterate generated by A-SymHALS  during the  update from $\mW_{ k} = (\mU_{k},\mV_{k})$ to $\mW_{ k+1} = (\mU_{k+1},\mV_{k+1})$, with $\mW_{ k}^0 = \mW_{ k}$ and $\mW_{ k}^L = \mW_{ k+1}$. We designate $\vu_{i,k}^{j}$ and $\vv_{i,k}^{j}$ the $i$-th columns of $\mU_k^{j}$ and $\mV_k^{j}$, respectively. Correspondingly, $\vu_{i,k}$ and $\vv_{i,k}$ represent the $i$-th columns of $\mU_k$ and $\mV_k$, respectively. At the $k$-th outer iteration  and $j+1$-th inner iteration of $\vu_{i}$, we denote
\[
g_{i,k}^{j+1} (\vu_i) = g(\vu_{1,k}^{j+1},\cdots, \vu_{i-1,k}^{j+1}, \vu_i,  \vu_{i+1,k}^{j},\cdots, \vu_{r,k}^{j}, \mV_{k})
\]
as the function $g$ when restricted to block $\vu_i$.
We will also rewrite \eqref{eq:SNMF by reg} as a unconstrained optimization problem by using the function 
\[
   f(\mU,\mV) = g(\mU,\mV) + \sigma_{+}(\mU) + \sigma_{+}(\mV),
\]
with $\sigma_{+}$ being the indicator function of the nonnegative constraint.

\subsection{Definitions and Basic Ingredients}
Before going to the main proof, we  first introduce some supporting materials.

\begin{defi}[first order optimality]\label{def:optimality condition}
    A point $\mW^\star = (\mU^\star,\mV^\star)$ is called a critical point of problem \eqref{eq:SNMF by reg} if it satisfies 
    \[
    \begin{split}
           0 \in \partial f(\mW^\star) &= \left( \nabla_{\mU} g(\mU^\star,\mV^\star), \nabla_{\mV} g(\mU^\star,\mV^\star) \right) \\
           &\quad + \left(\partial \sigma_{+}(\mU^\star),  \partial \sigma_{+}(\mV^\star) \right), 
    \end{split}
    \]
    where $\partial \sigma_{+}$ represents the usual convex subdifferential; i.e., $\partial \sigma_{+}(\mU):= \{\mS \in \R^{n\times r}: \langle \mS, \mU' - \mU\rangle \leq 0, \forall \ \mU'\in \R^{n\times r} \}$.
\end{defi}

The following property states the geometry of a function $h$ around its critical points, which plays a key role in our sequel analysis.
\begin{defi}[KL property]\cite{bolte2007lojasiewicz,attouch2009convergence}\label{def:KL}
	We say a proper semi-continuous function $h(\vu)$ satisfies Kurdyka-Lojasiewicz (KL) property if for every  critical point $\overline{\vu}$ of $h(\vu)$, there exist $\delta>0,~\eta>0,~\theta\in[0,1),~C_1>0$ such that for all 
\[
   \vu \in B(\overline \vu,\delta) \cap \{ \vu: h(\overline \vu) < h(\vu) <  h(\overline \vu) + \eta\}, 
\]
one has
	\[
	\left|h(\vu) - h(\overline{\vu})\right|^{\theta} \leq C_1 \dist(\vzero, \partial h(\vu)),
	\]
where $B(\overline \vu,\delta) := \{\vu: \|\vu - \overline \vu\|_2\leq \delta\}$.	
\end{defi}
The above KL property (also known as KL inequality) states the regularity of $h$ around its critical point $\overline \vu$.     \cite[Section 4]{bolte2014proximal} shows that our function of interest   $f$ satisfies this property.   Indeed, the KL property is general enough such that a large class of functions hold such a property,   including but never limited to any polynomial, any norm, any quasi norm, $\ell_0$-norm, indicator function of smooth manifold, etc; see \cite{bolte2014proximal,attouch2013convergence} for more discussions and examples.

\subsection{Supporting Results} \label{subsec:supporting results} 

Compared with the standard sufficient decrease property  of descent algorithms \cite{bolte2014proximal,attouch2010proximal}, the following lemma states the strengthened sufficient decrease property of A-SymHALS. 
\begin{lem}\label{lem:sufficient decrease 2}
For any $k\geq 0$, we have
	\e \label{eq:sufficient decrease 2}
	\begin{split} 
		&f(\mW_k) - f(\mW_{k+1}) \\ 
		&\geq \frac{\lambda}{4L}\left[ \|\mW_{k+1}  - \mW_{k}\|_F^2 
		+ \sum_{j=0}^{L-1} \|\mW_{k}^{j+1}  - \mW_{k}^j\|_F^2 \right].
	\end{split}
	\ee
\end{lem}

\begin{proof}[Proof of \Cref{lem:sufficient decrease 2}]
In the $j+1$-th inner iteration for updating $\mU_{k}$ to $\mU_{k+1}$,  suppose we update $\vu_i$, which amounts to solve the subproblem (see \Cref{subsec:notations in appendix} about notations)
\[
\min_{\vu_i} \ f_{i,k}^{j+1}(\vu_i)  =  g_{i,k}^{j+1}(\vu_i) + \sigma_{+}(\vu_i). 
\]
Since $f_{i,k}^{j+1}(\vu_i)$ is $\lambda$-strongly convex due to the penalty term $\frac{\lambda}{2} \|\vu_i - \vv_{i,k}\|_2^2$ of $g_{i,k}^{j+1}$, we have 
\[
    f_{i,k}^{j+1}(\vx)  \geq  f_{i,k}^{j+1}(\vy) + \langle \vs, \vx- \vy \rangle + \frac{\lambda }{2} \|\vx - \vy\|_2^2, 
\]
for all $\vx,\vy \in \R^{n\times 1}$ and $\vs \in \partial_{\vu_i} f_{i,k}^{j+1}(\vy)$. Substituting $\vx = \vu_{i,k}^j$ and $\vy = \vu_{i,k}^{j+1}$ yields
\[
   f_{i,k}^{j+1}(\vu_{i,k}^j)  \geq  f_{i,k}^{j+1}(\vu_{i,k}^{j+1}) + \frac{\lambda }{2} \|\vu_{i,k}^j - \vu_{i,k}^{j+1}\|_2^2, 
\]
where we also used  the optimality of $\vu_{i,k}^{j+1}$ in the subproblem which implies that $0\in \partial_{\vu_i} f_{i,k}^{j+1}(\vu_{i,k}^{j+1})$.  Upon summing both sides of the above inequality over $i = 1,\cdots,r$, one has
	\[
f(\mU_k^j,\mV_k) - f(\mU_{k}^{j+1},\mV_{k}) \geq \frac{\lambda}{2} \|\mU_{k}^{j+1}  - \mU_{k}^j\|_F^2,
\]
for all $j \in \{0, 1, \cdots, L-1\}$.  
Summing both sides of the above inequalities for $j$ from 0 to $L-1$ gives
	\[
	\begin{split}
	& f(\mU_k,\mV_k) - f(\mU_{k+1},\mV_{k})  = \sum_{j=0}^{L-1}  f(\mU_k^j,\mV_k) - f(\mU_{k}^{j+1},\mV_{k})\\
	&\geq \frac{\lambda}{4} \sum_{j=0}^{L-1} \|\mU_{k}^{j+1}  - \mU_{k}^j\|_F^2 + \frac{\lambda}{4} \sum_{j=0}^{L-1} \|\mU_{k}^{j+1}  - \mU_{k}^j\|_F^2\\
	&\stackrel{i}{\geq}  \frac{\lambda}{4} \frac{\left( \|\mU_{k}^{1}  - \mU_{k}^0\|_F +\cdots+\|\mU_{k}^{L}  - \mU_{k}^{L-1}\|_F   \right)^2}{L}  \\
	&\quad + \frac{\lambda}{4} \sum_{j=0}^{L-1} \|\mU_{k}^{j+1}  - \mU_{k}^j\|_F^2\\
	&\stackrel{ii}{\geq}   \frac{\lambda}{4L} \|\mU_{k+1}  - \mU_{k}\|_F^2  + \frac{\lambda}{4} \sum_{j=0}^{L-1} \|\mU_{k}^{j+1}  - \mU_{k}^j\|_F^2.
	\end{split}
	\]
where $(i)$ used the fact that  $\sqrt{\frac{a_1^2 + a_2^2 +\cdots + a_L^2}{L}}  \geq \frac{|a_1| + |a_2| +\cdots + |a_L|}{L}$, and $(ii)$ follows from the triangle inequality.  We complete the proof by using a similar argument to obtain 
 \[
 \begin{split}
 &f(\mU_{k+1},\mV_k) - f(\mU_{k+1},\mV_{k+1})  \\&\geq \frac{\lambda}{4L} \|\mV_{k+1}  - \mV_{k}\|_F^2  + \frac{\lambda}{4} \sum_{j=0}^{L-1} \|\mV_{k}^{j+1}  - \mV_{k}^j\|_F^2 .
 \end{split}
 \]
 The desired result can be obtained by summing up the above two inequalities. 
\end{proof}

\Cref{lem:sufficient decrease 2} has the following direct result, which states the descent property of A-SymHALS. This fulfills the requirement of decreasing the objective function in \Cref{thm:critical point U = V}.   
\begin{lem}\label{lem:iterates regular}
	For any $k\geq 0$, we have:
	\begin{enumerate}[label=(\alph*)]
		\item The sequence $\{f(\mU_k,\mV_k)\}_{k\geq 0}$ of function values is monotonically decreasing and it converges to some finite value $f^\star\geq 0$:
		\[\lim\limits_{k\to\infty}f(\mU_k,\mV_k)=f^\star.\]
		\item The sequence $\{f(\mU_k,\mV_k)\}_{k\geq 0}$ is regular, i.e.,
		\e
		\lim_{k\rightarrow\infty} \| \mU_{k+1} -   \mU_{k} \|_F = 0,  \ \lim_{k\rightarrow\infty}\|   \mV_{k +1  } - \mV_{k}\|_F = 0.
		\label{eq:difference converges}
		\ee
		
	\end{enumerate}
\end{lem}

\Cref{lem:sufficient decrease 2} together with \Cref{lem:bound:iterate} also implies the boundness of the sequence $\{(\mW_k)\}_{k\geq 0} = \{(\mU_k,\mV_k)\}_{k\geq 0}$.
\begin{lem}\label{lem:bounded iterates}
	The sequence $\{(\mW_k)\}_{k\geq 0} = \{(\mU_k,\mV_k)\}_{k\geq 0}$ lies in a bounded subset.
\end{lem}

The following lemma estimates the local Lipschitz constant of the  gradient of function $g$ in \eqref{eq:SNMF by reg}. 
\begin{lem}
	The function $g(\mU,\mV) = \frac{1}{2}\|\mX - \mU\mV^\T\|_F^2 + \frac{\lambda}{2}\|\mU - \mV\|_F^2$ in \eqref{eq:SNMF by reg}
	has  Lipschitz continuous gradient with the Lipschitz constant as $2B+\lambda+\|\mX\|_F$
	in any bounded $\ell_2$-norm ball $\{(\mU,\mV): \|\mU\|_F^2+\|\mV\|_F^2\leq B\}$ for any  $B>0$.
	\label{lem:lipchitz}
\end{lem}
\begin{proof}
	To obtain the Lipschitz constant, it is equivalent to bound the spectral norm of the quadrature form of the Hessian $[\nabla^2 g(\mW)](\mD,\mD)$ for any $\mD:=(\mD_U,\mD_V)$:
	\begin{align*}
	&[\nabla^2 g(\mW)](\mD,\mD)\\
	&=	\|\mU\mD_V^\T+\mD_U\mV^\T\|_F^2-2\lg\mX,\mD_U\mD_V^\T\rg +\frac{\lambda}{2}\|\mD_V-\mD_U\|_F^2\\
	&\leq2\|\mU\|_F^2\|\mD_V\|_F^2+2\|\mV\|_F^2\|\mD_U\|_F^2\\ &\quad+\underbrace{\lambda\|\mD_U\|_F^2+\lambda\|\mD_V\|_F^2}_{=\lambda\|\mD\|_F^2}+2\|\mX\|_F\underbrace{\|\mD_U\mD_V^\T\|_F}_{\leq \|\mD\|_F^2/2}\\
	&\leq(2\|\mU\|_F^2+2\|\mV\|_F^2+\lambda+\|\mX\|_F)\|\mD\|_F^2 \\
	&\leq (2B+\lambda+\|\mX\|_F)\|\mD\|_F^2.
	\end{align*}
\end{proof}

As the iterate $\mW_k = (\mU_k,\mV_k)$ for all $k\geq 0$ lives in the $\ell_2$-norm ball with the radius $\sqrt{B_0}$ (see \eqref{eqn:bound} for definition of  $B_0$) according to \Cref{lem:bound:iterate} and \Cref{lem:iterates regular}, the function $g$ has Lipschitz continuous gradient with the Lipschitz constant being $2B_0+\lambda+\|\mX\|_F$ around each $\mW_k$.

\begin{lem}\label{lem:safeguard 2}
		For any $k\geq 0$, we have
	\e
	\begin{split} \label{eq:safeguard 2}
	&\dist(\mzero,\partial f(\mU_{k+1},\mV_{k +1 })) \\ 
	&\quad\leq 2r(2B_0+\lambda+\|\mX\|_F)\|\mW_{k+1}-\mW_{k}^{L-1}\|_F.
	\end{split}
	\ee
\end{lem}

\begin{proof}[Proof of \Cref{lem:safeguard 2}]
	The $i$-th block $\vu_{i,k+1}$ of $\mU_{k+1}$ is updated according to
	\[
	\vu_{i,k+1} = \vu_{i,k}^{L} = \argmin_{\vu} g_{i,k}^L (\vu) + \delta_{+}(\vu).
	\]
It then follows from the first order optimality that
	\[
	-  \nabla g_{i,k}^L (\vu_{i,k+1})  \in \partial \delta_{+}(	\vu_{i,k+1}),
	\]
which together with
	\begin{align*}
	\partial_{\vu_i} f(\mU_{k+1},\mV_{k +1 }) = \nabla_{\vu_i} g(\mU_{k+1},\mV_{k +1 })  + \partial \delta_{+}(	\vu_{i,k+1})
	\end{align*}
gives
	\[
	    \nabla_{\vu_i} g(\mU_{k+1},\mV_{k +1 })  -   \nabla g_{i,k}^L (\vu_{i,k+1}) \in \partial_{\vu_i} f(\mU_{k+1},\mV_{k +1 }). 
	\]
Similar result holds for $\partial_{\vv_i} f(\mU_{k+1},\mV_{k +1 })$. Now invoke the Lipschitz gradient condition of function $g(\mU,\mV)$ in \Cref{lem:lipchitz}:
	\begin{align*}
	&\dist(\mzero,\partial f(\mU_{k+1},\mV_{k +1 }))  \leq \sum_{i=1}^{r}  \dist(\vzero,\partial_{\vu_i} f(\mU_{k+1},\mV_{k +1 })) \\
	&\qquad+ \sum_{i=1}^{r}  \dist(\vzero,\partial_{\vv_i} f(\mU_{k+1},\mV_{k +1 })) \\
	&\leq 2r(2B_0+\lambda+\|\mX\|_F)\|\mW_{k+1}-\mW_{k}^{L-1}\|_F.
	\end{align*}	
\end{proof}

We denote $\calC(\mW_0)$ as the collection of all the limit points  of  the sequence $\{\mW_k\}$ (which may depend on the initialization $\mW_0$). The following lemma provides some useful properties and optimality of  $\calC(\mW_0)$.

\begin{lem}\label{prop:function value converges}
 $f$ is constant on $\calC(\mW_0)$ and
	\[
	\lim\limits_{k\rightarrow\infty} f(\mU_k,\mV_k) = f(\mU^\star,\mV^\star), \ \ \forall \ (\mU^\star,\mV^\star) \in \calC(\mW_0).
	\]
\end{lem}
\begin{proof}[Proof of \Cref{prop:function value converges}]
	According to \Cref{lem:bounded iterates}, we can extract an arbitrary convergent subsequence $\{\mW_{k_m}\}_{m\geq0}$  which converges to $\mW^\star\in \calC(\mW_0)$. By the definition of the algorithm  we have
	\[
	\mU_{k_m} \geq 0, \ \mV_{k_m} \geq 0, \ \ \forall \ k_m\geq 0.
	\]
	Thus,
	\[
	\lim\limits_{m\rightarrow \infty}  \delta_{+} (\mU_{k_m}) = 0, \ \lim\limits_{m\rightarrow \infty}  \delta_{+} (\mV_{k_m}) = 0.
	\]
	We now take limit on the subsequence $\{\mW_{k_m}\}_m$:
\[
	\lim\limits_{m\rightarrow \infty}   f(\mW_{k_m}) =     g(\lim_{m\rightarrow \infty}\mW_{k_m}) = g(\mW^\star),
	\]
	where we have used the continuity of the smooth part $g(\mW)$ in \eqref{eq:SNMF by reg}.
	Then from  \Cref{lem:iterates regular} we know that  $\{f(\mW_k)\}_{k\geq 0}$ forms a convergent sequence. The proof is completed by noting that for any convergent sequence, all its  subsequence must converge to an unique limiting point.
\end{proof}

\begin{lem}\label{lem:limit points set}
	Each element of $ \calC(\mW^0)$ is a critical point of \eqref{eq:SNMF by reg} and $\calC(\mW^0)$ is a nonempty, compact, and connected set with
	\[\lim\limits_{k\to\infty}\dist(\mW_k,\calC(\mW_0))=0.\]
	
\end{lem}
\begin{proof}[Proof of \Cref{lem:limit points set}]
	Let $\mS_{k} $ and $\mD_{k} $ be defined in  \Cref{lem:safeguard 2}. From \Cref{lem:iterates regular}, we have $\|\mW_{k+1}-\mW_{k}\|_F \rightarrow 0$. Hence 
	\[
	\lim\limits_{k\to\infty}(\mS_{k},\mD_{k})=\mzero.
	\]
	According to \Cref{lem:bounded iterates},  we can extract an arbitrary convergent subsequence $\{\mW_{k_m}\}_{m\geq0}$  with  limit $\mW^\star$.  Note that
	\[
	\mS_{k_m} = \nabla_{\mU}g(\mU_{k_m},\mV_{k_m}) + \overline\mS_{k_m}, \ \overline\mS_{k_m} \in \partial\delta_+(\mU_{k_m}).
	\]
	Since $\lim_{m \rightarrow \infty} \mS_{k_m} = \vzero$, $\lim_{m\rightarrow \infty} \mW_{k_m}= \mW^\star$, and $\nabla_{\mU}g$ is continuous, $\{\overline \mS_{k_m}\}$ is convergent. Denote by $\overline \mS^\star = \lim_{m \rightarrow \infty} \overline\mS_{k_m}$. By the definition of $\overline\mS_{k_m} \in \partial\delta_+(\mU_{k_m})$, for any $\mU'\in\R^{n\times r}$, we have
	\[
	\delta_+(\mU') - \delta_+(\mU_{k_m})\geq \langle \overline\mS_{k_m},\mU' - \mU_{k_m} \rangle.
	\]
Due to $\lim_{m\rightarrow \infty}  \delta_{+} (\mU_{k_m}) =0 = \delta_{+} (\mU^\star)$ (since $\mU_{k_m}\geq 0$), taking $m\rightarrow \infty$ for both sides of the above equation gives
	\[
	\delta_+(\mU') - \delta_+(\mU^\star)\geq \langle \overline\mS^\star,\mU' - \mU^\star \rangle.
	\]
As the above equation holds for any $\mU'\in\R^{n\times r}$, we have $\overline \mS^\star \in \partial \delta_+(\mU^\star)$ and $\vzero = \nabla_{\mU}g(\mU^\star,\mV^\star) + \overline\mS^\star\in \partial_{\mU}f(\mW^\star)$. With similar argument, we get $\vzero \in \partial_{\mV}f(\mW^\star)$ and thus
	\[
	\mzero \in \partial f(\mW^\star),
	\]
	which implies that $\mW^\star$ is a critical point of \eqref{eq:SNMF by reg}.

	Finally, by \cite[Lemma 5]{bolte2014proximal} and identifying that the sequence $\{\mW_k\}$ is bounded and regular (i.e. $\lim_{k\to\infty}\|\mW_{k+1}-\mW_k\|_F=0$), we conclude that the set of limit points $\calC(\mW_0)$ is a nonempty, compact, and connect set satisfying
	$$\lim\limits_{k\to\infty}\dist(\mW_k,\calC(\mW_0))=0.$$
\end{proof}

With the KL property of $f$ (see \Cref{def:KL}),   \Cref{prop:function value converges}, and \Cref{lem:limit points set}, we have the following \emph{uniform} KL property of $f$ on the set $\calC(\mW_0)$ by following the argument of \cite[Lemma 6]{bolte2014proximal}. 
\begin{lem}\label{lem:Uniform KL}
	There exist a set of  uniform constants $C_2>0, \ \delta>0, \ \eta>0$ and $ \theta\in[0,1)$ such that for all $ \mW^\star \in \calC(\mW_0)$ and $\mW$ in the following intersection
\[
    B(\calC(\mW_0),\delta) \cap \{ \mW: f(\mW^\star) < f(\mW) <  f(\mW^\star) + \eta\},
\]
we have
	\[
	\left|f(\mW) - f(\mW^\star)\right|^{\theta} \leq C_2 \dist(\mzero, \partial f(\mW)).
	\]
\end{lem}

\subsection{Formal Proof of \Cref{thm:convergence} } \label{subsec:proof of A-SymHALS}
\begin{proof}[Proof of \Cref{thm:convergence}]
With all the intermediate techniques developed above, we are now going to complete the proof of \Cref{thm:convergence}; that is, showing that the sequence of iterates $\{\mW_k\}_{k\geq 0}$ generated by A-SymHALS is convergent and converges to a critical point  $\mW^\star$ of \eqref{eq:SNMF by reg}. 

Recall \Cref{lem:limit points set} that
 $\lim_{k\to\infty}\dist(\mW_k,\calC(\mW_0))=0$ and \Cref{prop:function value converges} that $\lim_{k\to\infty}f(\mW_k)=f(\mW^\star), \forall \mW^\star\in \calC(\mW_0)$. For any fixed $\delta>0,\eta>0$, there exists $k_0>0$ such that $\dist\left(\mW_k,\calC(\mW_0)\right) \leq \delta$ and $f(\mW_k) < f(\mW^\star) +\eta$ for all $k\geq k_0$. Furthermore, from \Cref{lem:iterates regular}, we have $f(\mW_{ k}) > f(\mW^\star)$  for all $k\geq 0$.  Hence, from \Cref{lem:Uniform KL} one has
\e\label{eq:ukl}
\left[f(\mW_k) - f(\mW^\star)\right]^\theta \leq  C_3 \dist(\mzero, \partial f(\mW_k)) \ \forall k\geq k_0.
\ee
In the subsequent analysis, we restrict to $k\geq k_0$.
Construct a concave function  $x^{1-\theta}$ for some $\theta\in[0,1)$ with domain $x>0$. Obviously, by the concavity, we have
\[  x_2^{1-\theta}-x_1^{1-\theta}\geq   (1-\theta) x_2^{-\theta}(x_2-x_1), \forall x_1>0,x_2>0.\]
By replacing $x_1$ by $f(\mW_{k+1}) - f(\mW^\star)$ and  $x_2$ by $f(\mW_{k}) - f(\mW^\star)$ and using the sufficient decrease property in \Cref{lem:sufficient decrease 2}, we have
 (we will hide all absolute and independent constants in $\overline C$ to simplify notation)
\begin{equation} \label{eq:convergence 1}
\begin{split}
& \left(f(\mW_{k}) - f(\mW^\star)\right)^{1-\theta} - \left(f(\mW_{k+1}) - f(\mW^\star )\right)^{1-\theta}\\
&\geq  \overline C\frac{f(\mW_{k}) - f(\mW_{k+1}) }{\left(f(\mW_{k}) - f(\mW^\star)\right)^\theta}.\\
&\geq \overline C \frac{\|\mW_{k+1}  - \mW_{k}\|_F^2 + \sum_{j=0}^{L-1} \|\mW_{k}^{j+1}  - \mW_{k}^j\|_F^2}{\dist(\mzero,\partial f(\mW_k))}\\
&\geq\overline C \frac{\|\mW_{k+1}  - \mW_{k}\|_F^2 + \sum_{j=0}^{L-1} \|\mW_{k}^{j+1}  - \mW_{k}^j\|_F^2}{\|\mW_{k}-\mW_{k-1}\|_F+  \|\mW_{k}-\mW_{k-1}^{L-1}\|_F}\\
&\geq\overline C \frac{\|\mW_{k+1}  - \mW_{k}\|_F^2 +  \|\mW_{k+1}  - \mW_{k}^{L-1}\|_F^2}{\sqrt{\|\mW_{k}-\mW_{k-1}\|_F^2+  \|\mW_{k}-\mW_{k-1}^{L-1}\|_F^2}},
\end{split}
\end{equation}
where the second inequality uses \Cref{lem:sufficient decrease 2} and \eqref{eq:ukl}, the third inequality is due to \Cref{lem:safeguard 2}, and $\overline C$ in the last line is
\[
\overline C =  \frac{\lambda(1-\theta)}{8lrC_3(2B_0+\lambda+\|\mX\|_F)} >0.
\]
Due to $\frac{a}{t} =  \frac{a}{t} +t -t \geq 2 \sqrt{a} - t$ for all $a\ge 0,t> 0$, we have
\begin{equation} \label{eq:convergence 2}
\begin{split}
&\frac{\|\mW_{k+1}  - \mW_{k}\|_F^2 +  \|\mW_{k+1}  - \mW_{k}^{L-1}\|_F^2}{\sqrt{\|\mW_{k}-\mW_{k-1}\|_F^2+  \|\mW_{k}-\mW_{k-1}^{L-1}\|_F^2}}  \\
&\geq 2\sqrt{\|\mW_{k+1}-\mW_{k}\|_F^2+  \|\mW_{k+1}-\mW_{k}^{L-1}\|_F^2}\\
&\quad -  \sqrt{\|\mW_{k}-\mW_{k-1}\|_F^2+  \|\mW_{k}-\mW_{k-1}^{L-1}\|_F^2}
\end{split}
\end{equation}
Combining \eqref{eq:convergence 1} and \eqref{eq:convergence 2} and summing them up from $\widetilde k\ge k_0$ to $m\rightarrow \infty$ yields
\e \label{eq:difference summable 2}
\begin{split}
	&\sum_{k=\widetilde k}^{\infty}  \sqrt{\|\mW_{k+1}-\mW_{k}\|_F^2 +  \|\mW_{k+1}-\mW_{k}^{L-1}\|_F^2} \\
	& \leq  \sqrt{\|\mW_{\widetilde k} - \mW_{\widetilde k-1}\|_F^2 + \|\mW_{\widetilde k} - \mW_{\widetilde k-1}^{L-1}\|_F^2}\\
	&\quad   + \frac{1}{\overline C}  \left(f(\mW_{\widetilde k}) - f(\mW^\star)\right)^{1-\theta},
\end{split}
\ee
which immediately implies
\[
\sum_{k=\widetilde k}^{\infty}  \|\mW_{k+1}-\mW_{k}\|_F <\infty.
\]
Thus, we conclude that   $\{\mW_k\}_{k\geq 0}$ is a Cauchy sequence and hence is convergent.  It immediately follows that the limit points set $\calC(\mW_0) = \{\mW^\star\}$ is a singleton and $\mW^\star = (\mU^\star,\mV^\star)$ is a critical point of \eqref{eq:SNMF by reg} due to \Cref{lem:limit points set}.

As for convergence rate, it follows from \eqref{eq:difference summable 2}  that
\e
\begin{split}
	&\|\mW_{\widetilde k}-\mW^\star\|_F \\
	&\leq \sum_{k=\widetilde k}^{\infty}  \sqrt{\|\mW_{k+1}-\mW_{k}\|_F^2 +  \|\mW_{k+1}-\mW_{k}^{L-1}\|_F^2} \\
	& \leq  \sqrt{\|\mW_{\widetilde k} - \mW_{\widetilde k-1}\|_F^2 + \|\mW_{\widetilde k} - \mW_{\widetilde k-1}^{L-1}\|_F^2} \\
	&  + \alpha \left( \sqrt{\|\mW_{\widetilde k} - \mW_{\widetilde k-1}\|_F^2 + \|\mW_{\widetilde k} - \mW_{\widetilde k-1}^{L-1}\|_F^2} \right)^{\frac{1-\theta}{\theta}}
\end{split}
\ee
for some constant $\alpha >0$, where the last line uses \eqref{eq:ukl} and \Cref{lem:safeguard 2}. Denoting by 
 \[
 P_{\widetilde k}  = \sum_{k=\widetilde k}^{\infty}  \sqrt{\|\mW_{k+1}-\mW_{k}\|_F^2 +  \|\mW_{k+1}-\mW_{k}^{L-1}\|_F^2},
 \]
we obtain
\e
P_{\widetilde k} \leq P_{{\widetilde k}-1}- P_{\widetilde k} + \alpha \left[P_{\widetilde k-1}- P_{\widetilde k}\right]^{\frac{1-\theta}{\theta}}. \label{eq:convergence rate 3}
\ee
The above recursion about the sequence $\{P_{\widetilde k} \}_{\widetilde k \geq k_0}$
 has exactly the same form with \cite[eq. (12)]{attouch2009convergence}. Hence the convergence rate can be obtained by following the same arguments after \cite[eq. (12)]{attouch2009convergence}. This completes the proof. 
\end{proof}

\section{Conclusion}
In order to design efficient alternating-type algorithms for the symmetric NMF, we transfer this problem to a penalized nonsymmetric NMF. We have proved that solving the nonsymmetric reformulation returns a solution for the original symmetric NMF when the penalty term is relatively large, in sharp contrast to the classical result for the methods of Lagrangian multiplier that suggests it happens only when the penalty term tends to infinity. Furthermore, we have proved that various alternating-type algorithms---when used to efficiently solve the  nonsymmetric reformulation---admit strong convergence guarantee in the sense that the generated sequence is convergent at least at a sublinear rate and it converges to a critical point of the original symmetric NMF. \revise{An interesting question would be whether it is possible to further improve the lower bound on the penalty parameter $\lambda$ that ensures convergence, and even to the extreme case that whether the convergence to a critical point of the original symmetric NMF is guaranteed for any positive $\lambda$. In additions, it would also} be of great interest to extend both algorithmic strategy and theoretical guarantee for multidimensional cases, such as symmetric tensor factorization and symmetric nonnegative tensor factorization.

\section*{Acknowledgment}
We gratefully acknowledge Dr. Songtao Lu for sharing the code used in \cite{lu2017nonconvex}, and \revise{the four anonymous reviewers for their constructive comments.}  X. Li is partially supported by the National Natural Science Foundation of China (NSFC) grant NSFC-72150002 and by AC01202101037 and AC01202108001 from Shenzhen Institute of Artificial Intelligence and Robotics for Society (AIRS).       Z. Zhu is partially supported by the NSF grants CCF2106881 and  CCF 2008460.

\bibliographystyle{ieeetr}
\bibliography{Convergence}

\end{document}